\documentclass{article}

    \usepackage[final]{neurips_2024}

\usepackage[utf8]{inputenc} % allow utf-8 input
\usepackage[T1]{fontenc}    % use 8-bit T1 fonts
\usepackage{hyperref}       % hyperlinks
\usepackage{url}            % simple URL typesetting
\usepackage{booktabs}       % professional-quality tables
\usepackage{amsfonts}       % blackboard math symbols
\usepackage{nicefrac}       % compact symbols for 1/2, etc.
\usepackage{microtype}      % microtypography
\usepackage{xcolor}         % colors

\usepackage[T1]{fontenc}    % use 8-bit T1 fonts
\usepackage{amsmath}
\usepackage{amssymb}
\usepackage{mathtools}
\usepackage{amsthm}
\usepackage{makecell}
\usepackage{amsmath,amsfonts,amsthm,amssymb,amsopn,bm,color,mathtools,booktabs,multirow,appendix,caption}

\usepackage[textsize=tiny]{todonotes}

\usepackage{algpseudocode}
\usepackage{algorithm}
\usepackage{subcaption}

\theoremstyle{plain}
\newtheorem{theorem}{Theorem}[section]

\newtheorem{lemma}[theorem]{Lemma}

\theoremstyle{definition}

\theoremstyle{remark}

\usepackage[textsize=tiny]{todonotes}

\def\E{\mathbb{E}}
\def\P{\mathbb{P}}
\def\R{\mathbb{R}}
\def\1{\bm{1}}
\newcommand{\mb}[1]{\mathbf{#1}}
\newcommand{\mc}[1]{\mathcal{#1}}

\renewcommand{\P}{\mathbb P}

\DeclareMathOperator*{\argmax}{arg\,max}

\newenvironment{sketch}{\paragraph{\bf Proof Sketch}}{\hfill$\square$}
\newcommand{\subucb}{\textsc{Sub-UCB }}
\newcommand{\reg}{R_\textbf{gr}}

\title{Nearly Minimax Optimal Submodular Maximization with Bandit Feedback}

\author{
Artin Tajdini, Lalit Jain, Kevin Jamieson\\
University of Washington, Seattle, WA\\
\texttt{\{artin, jamieson\}@cs.washington.edu, lalitj@uw.edu}}

\begin{document}

\maketitle

\begin{abstract}
    We consider maximizing an unknown monotonic, submodular set function $f: 2^{[n]} \rightarrow [0,1]$ with cardinality constraint under stochastic bandit feedback. At each time $t=1,\dots,T$ the learner chooses a set $S_t \subset [n]$ with $|S_t| \leq k$ and receives reward $f(S_t) + \eta_t$ where $\eta_t$ is mean-zero sub-Gaussian noise. The objective is to minimize the learner's regret with respect to an approximation of the maximum $f(S_*)$ with $|S_*| = k$, obtained through robust greedy maximization of $f$. To date, the best regret bound in the literature scales as $k n^{1/3} T^{2/3}$. And by trivially treating every set as a unique arm one deduces that $\sqrt{ {n \choose k} T }$ is also achievable using standard multi-armed bandit algorithms. In this work, we establish the first minimax lower bound for this setting that scales like $\tilde{\Omega}(\min_{L \le k}(L^{1/3}n^{1/3}T^{2/3} + \sqrt{{n \choose k - L}T}))$. For a slightly restricted algorithm class, we prove a stronger regret lower bound of $\tilde{\Omega}(\min_{L \le k}(Ln^{1/3}T^{2/3} + \sqrt{{n \choose k - L}T}))$. Moreover, we propose an algorithm Sub-UCB that achieves regret $\tilde{\mathcal{O}}(\min_{L \le k}(Ln^{1/3}T^{2/3} + \sqrt{{n \choose k - L}T}))$ capable of matching the lower bound on regret for the restricted class up to logarithmic factors.
\end{abstract}

\section{INTRODUCTION}
Optimizing over sets of $n$ ground items given noisy feedback is a common problem.
For example, when a patient comes into the hospital with sepsis (bacterial infection of the blood), it is common for a cocktail of $1 < k \leq n$ antibiotics to be prescribed.
This can be attractive for reasons including 1) the set could be as effective (or more) than a single drug alone, but each unit of the cocktail could be administered at a far lower dosage to avoid toxicity, or 2) could be more robust to resistance by blocking a number of different pathways that would have to be overcome simultaneously, or 3) could cover a larger set of pathogens present in the population. 
In this setting the prescriber wants to balance exploration with exploitation over different subsets to maximize the number of patients that survive. 
As a second example, we consider factorial optimization of web-layouts: you have $n$ pieces of content and $k$ locations on the webpage to place them--how do you choose subsets to maximize metrics like click-through rate or engagement?

Given there are $\approx n^k$ ways to choose $k$ items amongst a set of $n$, this optimization problem is daunting. 
It is further complicated by the fact that for any set $S_t$ that we evaluate at time $t$, we only get to observe a noisy realization of $f$, namely $y_t = f(S_t) + \eta_t$ where $\eta_t$ is mean-zero, sub-Gaussian noise.
In the antibiotics case, this could be a Bernoulli indicating whether the patient recovered or not, and in the web-layout case this could be a Bernoulli indicating a click or a (clipped) real number to represent the engagement time on the website.
To make this problem more tractable, practitioners make structural assumptions about $f$.
A common assumption is to assume that higher-order interaction terms are negligible \cite{hill2017efficient,chen_adversarial_2021}. 
For example, assuming only interactions up to the second degree would mean that there exist parameters $\theta^{(0)} \in \R$, $\theta^{(1)} \in \R^n$, and $\theta^{(2)} \in \R^{\binom{n}{2}}$ such that
\begin{align}
    f(S) = \theta^{(0)} + \sum_{i \in S} \theta^{(1)}_i + \sum_{i,j \in S, i \neq j} \theta_{i,j}^{(2)} \label{eqn:multivariate}. 
\end{align}
However, this model can be very restrictive and even if true, the number of unknowns scales like $n^2$ which could still be intractably large. 

An alternative strategy is to remain within a non-parametric class, but reduce our ambitions to measuring performance relative to a different benchmark which is easier to optimize. 
We say a set function $f: 2^{[n]} \rightarrow \R$ is \textit{increasing and submodular} if for all $A \subset B \subset [n]$ we have $f(A) \leq f(B)$ and
\begin{align}
    f(A \cup B) + f(A \cap B) \leq f(A) + f(B) \label{eqn:submodular}.
\end{align}
Such a condition limits how quickly $f$ can grow and captures some notion of diminishing returns. 
Diminishing returns is reasonable in both the antibiotics and webpage optimization examples.
% and justified empirically in Figure~\ref{}.
It is instructive to note that a sufficient condition for the parametric form of \eqref{eqn:multivariate} to be submodular is for $\max_{i,j}\theta_{i,j}^{(2)} \leq 0$.
But in general, $f$ still has $\approx n^k$ degrees of freedom even if it is monotonic and submodular.
And it is known that for an unknown $f$, identifying $S^* := \arg\max_{S \subset [n]: |S|=k} f(S)$ may require evaluating $f$ as many as $n^k$ times.

The power of submodularity is made apparent through the famous result of \cite{nemhauser1978submodular} which showed that the \textit{greedy algorithm} which grows a set one item at a time by adding the item with the highest marginal gain returns a solution that is within a $(1 - e^{-1})$-multiplicative factor of the optimal solution.
That is, if we begin with $S_{gr}^f=\emptyset$ and set $S_{gr}^f \leftarrow \arg\max_{i \in [n]\setminus S_{gr}^f} f(S_{gr}^f \cup \{i\})$ until $|S_{gr}^f|=k$, then $f(S_{gr}^f) \geq (1-1/e) f(S_*^f)$ where $S_*^f := \arg\max_{S \in [n] : |S|\leq k} f(S)$ if $f$ is increasing and submodular. 
This result is complemented by \cite{10.1145/285055.285059} which shows achieving any $(1 - e^{-1} + \epsilon)$-approximation is NP-Hard. 
Under additional assumptions like curvature, this guarantee can be strengthened.

Due to the centrality of the greedily constructed set to the optimization of a submodular function, it is natural to define a performance measure relative to the greedily constructed set. 
However, as discussed at length in the next section, because we only observe noisy observations of the underlying function, recovering the set constructed greedily from noiseless evaluations is too much to hope for.
Consequently, there is a more natural notion of regret against a noisy greedy solution, denoted $\reg$, that actually appears in the proofs of all upper bounds found in the literature for this setting (see the next section for a definition).
%\textcolor{red}{However, these papers main theorems typically report an immediate upper bound on $\reg$, known as $\alpha$-regret. }

%The notion is also used in proves of online nonmonotone submodular maximization literature (where the offline greedy solution yields an $1/2$-approximation which is robust to additive noise similar to the monotone setting). In particular, in proofs of \cite{}
For this notion of regret, previous works have demonstrated that a regret bound of $\reg = O( \text{poly}(k) n^{1/3} T^{2/3})$ is achievable (\cite{nie_explore-then-commit_nodate}, \cite{streeter_online_2007}). 
% \kevin{cite a representative work here}
This $T^{2/3}$ rate is unusual in multi-armed bandits, where frequently we expect a regret bound to scale as $T^{1/2}$.
On the other hand, by treating each $k$-subset as a separate arm, one can easily adapt existing algorithms to achieve a regret bound of $\sqrt{ \binom{n}{k} T }$.
This leads to the following question:
\begin{quote}
    \emph{Does there exist an algorithm that obtains $\sqrt{n^r T}$ regret for $r = o(k)$ on every instance? And if not, what is the optimal dependence on $k$ and $n$ for a bound scaling like $T^{2/3}$?}
    % relative to the greedy benchmark set $S_{gr}^{f, \epsilon}$? And if not, can we improve upon $R(f,S_{gr}^{f, \epsilon},T) + T\mb{1}^T\boldsymbol{\epsilon} \leq \text{poly}(k) n^{1/3} T^{2/3}$?}
\end{quote}
% To both of these questions this papers answers is partially positive, depending on the values of $n$, $k$, and $T$. 
To address these questions, we prove a minimax lower bound and complement the result with an algorithm achieving a matching upper bound. 
To be precise, the contributions of this paper include:
\begin{itemize}
    \item A minimax lower bound demonstrating that $\reg = \tilde{\Omega} \Big(\min_{0\leq 
L\leq k}(L^{1/3}n^{1/3}T^{2/3} + \sqrt{{n \choose k - L}T})\Big)$. 
    In words, for small $T$, a $T^{2/3}$ regret bound is inevitable, for large $T$ the $\sqrt{\binom{n}{k} T}$ bound is optimal, with an interpolating regret bound for in between. 
    \begin{itemize}
        \item For slightly restricted class of algorithms with non-adaptive greedy error threshold, we have the improved $\reg = \tilde{\Omega}( \Big(\min_{0\leq L\leq k}(Ln^{1/3}T^{2/3} + \sqrt{{n \choose k - L}T})\Big)$.
    \end{itemize}
    \item We propose an algorithm that for any increasing, submodular $f$, we have $\reg = \tilde{\mathcal{O}} \min_{0\leq L\leq k}(Ln^{1/3}T^{2/3} + \sqrt{{n \choose k - L}T})$.
    As this matches our lower bound, we conclude that this is the first provably tight algorithm for optimizing increasing, submodular functions with bandit feedback. Existing algorithms construct a set by greedily adding $k$ items. Our main insight is that it is actually optimal to build up a set up to a size $i^{\ast}$ and then for the remaining stages play sets of size $k$ that include the initial set of size $i^{\ast}$. Our choice of $i^{\ast}$ is directly motivated by our lower bound. 
    
%    In fact, to achieve the lower bound, we show a simple modification of existing algorithms 
\end{itemize}
In what remains, we will formally define the problem, discuss the related work, and then move on to the statement of the main theoretical results. 
Experiments and conclusions follow.  

\subsection{Problem Statement}

% In this section we formally define the problem in this paper. 
Let $[n] = \{1,\dots,n\}$ denote the set of base arms, $T$ be the time horizon, and $k$ be a given cardinality constraint.
At time $t$, the agent selects a set $S_t \subset [n]$ where $|S_t| \le k$, and observes reward $f(S_t) + \eta_t$ where $\eta_t$ is i.i.d. mean-zero 1-sub-Gaussian noise, and $f:2^{[n]} \rightarrow [0,1]$ is an unknown monotone non-decreasing submodular function defined for all sets of cardinality at most $k$. 

Ideally, our goal would be to minimize the regret relative to pulling the best set $S^*:= \argmax_{|S| \le k} f(S)$ at each time. 
In general, even if we had the ability to evaluate the true function $f(\cdot)$ (i.e. without noise), maximizing a submodular function with a cardinality constraint is NP-hard.  
However, greedy algorithms which sequentially add points, i.e. $S^{(i+1)} = \arg\max_{a\not\in S^{(i)}} f(S^{(i)}\cup a), 1\leq i\leq k$ guarantee that $f(S^{(k)}) \geq \alpha f(S^{\star})$ with $\alpha \geq 1-1/e$ in worst-case. Unfortunately, since we do not know $f(\cdot)$ and instead only have access to noisy observations, running the greedy algorithm on any estimate $\hat{f}(\cdot)$ may not necessarily guarantee an $\alpha=(1-1/e)$-approximation to $f(S^{\ast})$\footnote{The gap between maximum gain and rest of the elements in the greedy path for lower cardinalities can be arbitrary small, making them indistinguishable with $T$ queries. Therefore, only making queries to sets of size $k$ would give any information on the greedy solution.}. 

%While the greedy set has this great guarantee, even it is unattainable because we can only observe \emph{noisy} estimates of $f(S)$ making it impossible to guarantee exact greedy set.
% Also, define regret relative to a set $S$ by
% \[
% R(S) := \sum_{t = 1}^{T} f(S) - f(S_t)
% \]
% We now discuss why this notion is relevant. 
%\kevin{Move this first lemma to appendix, and the proof of the escond lemma to the appendix. Motivate the lemma a bit}

Consequently, a natural notion to address noisy observations is an $\boldsymbol{\epsilon}$-approximate greedy set for $\boldsymbol{\epsilon}\in [0,1]^k$. We define the following collection of sets of size $k$
\begin{align*}
    \mc{S}^{k, \boldsymbol{\epsilon}} = \{&S =S^{(k)}\supset \cdots\supset S^{(1)}, |S^{(i)}| = i, \\ &\max_{a \notin S^{(i)}} f(S^{(i)} \cup \{a\}) - f(S^{(i + 1)}) \le \epsilon_i\}.
\end{align*}
Intuitively, any $S\in \mc{S}^{k,\boldsymbol{\epsilon}}$ can be thought of as being constructed from a process that adds an element at stage $i$ which is $\epsilon_i$-optimal compared to the Greedy algorithm run on $f$. 
Such a set naturally arises as the output of the Greedy algorithm run on an approximation $\hat{f}$. 
% The main algorithm of this paper Algorithm~\ref{alg:cap} will be such a method. 
This set enjoys the following guarantee.

%We initialize $S^{(0)}$ as the null set and at each stage we add $x_{(i+1)} \in \{x\not \in S^{(i)}: \max_{} f(S^(i)\cup \{x\}) - f()\}$

\begin{lemma}(Theorem 6 in \cite{streeter_online_2007})
    \label{lemma:approx-greedy}
    %Define $S^{k, \epsilon}_\textbf{gr} := \argmin_{S^{(1)} \subset \dots \subset S^{(k)}: \epsilon\text{-good arm greedy}} f(S^{(k)})$, i.e. $\max_{a \notin S^{(i)}} f(S^{(i)} \cup \{a\}) - f(S^{(i + 1)}) \le \epsilon$. 
    For any $\boldsymbol{\epsilon} \ge \mb{0} \in \mathbb{R}^k$, and $S^{k, \boldsymbol{\epsilon}}_\textbf{gr}\in \mc{S}^{k,\boldsymbol{\epsilon}}$, we have 
    % \[f(S^{k, \epsilon}_\textbf{gr}) + \mb{1}^T\boldsymbol{\epsilon} \ge \frac{1}{c}(1 - e^{-c})f(S^{\ast})\] 
    \[f(S^{k, \boldsymbol{\epsilon}}_\textbf{gr}) + \mb{1}^T\boldsymbol{\epsilon} \ge (1 - e^{-1})f(S^{\ast}).\] 
    % where $c:= 1 - \min_{S, a \notin S} \frac{f(S \cup \{a\}) - f(S)}{f(a)}$ is the total submodular curvature of $f$.
\end{lemma}

% \lalit{Define $c$.}
% Given the above context, we are now ready to define the notion of regret that we wish to minimize. 
% Define regret relative to a set $S$ by
% \[
% R(S) := \sum_{t = 1}^{T} f(S) - f(S_t).
% \]
Lemma \ref{lemma:approx-greedy} is a noise-robust analogous result to the approximation ratio of the perfect greedy algorithm of \cite{nemhauser1978submodular} that says $f(S^{k, 0}_\textbf{gr}) \geq (1-e^{-1}) f(S^*)$.
Note that $|\mc{S}^{k,\boldsymbol{\epsilon}}|$ is non-decreasing in $\epsilon_i$ for all $i \in [k]$, so identifying a set in $\mc{S}^{k,\boldsymbol{\epsilon}}$ is in some sense easier for a larger $\mathbf{1}^T\boldsymbol{\epsilon}$. 
Thus, to define an appropriate definition of regret, the measure must balance the facts that comparing with the noiseless greedy approximation in $\mc{S}^{k,\mathbf{0}}$ may be impossible, but should account for identifying a set in $\mc{S}^{k,\boldsymbol{\epsilon}}$ is strictly easier for larger $\mathbf{1}^T\boldsymbol{\epsilon}$. 
% It says that given the slack of $ \mb{1}^T\boldsymbol{\epsilon}$ to account for the noise, we 
% , and shows that $R(S^{k, \epsilon}_\textbf{gr}) + T\mb{1}^T\boldsymbol{\epsilon}  \ge R_{\frac{1}{c}(1 - e^{-c})}$ for any $\epsilon$. %while our minimax optimal bounds are stronger and bound $\min_{\epsilon \ge 0, S^{k, \epsilon}_\textbf{gr} \in \mc{S}^{k, \epsilon}}R(S^{k, \epsilon}_\textbf{gr}) + k \epsilon T$ regret directly. 
Inspired by the above lemma we define \textit{robust greedy regret} 
\begin{align}
    \reg := \min_{\boldsymbol{\epsilon} \ge \mathbf{0}, S^{k, \boldsymbol{\epsilon}}_\textbf{gr} \in \mc{S}^{k, \boldsymbol{\epsilon}}}R(S^{k, \boldsymbol{\epsilon}}_\textbf{gr}) + T\mb{1}^T\boldsymbol{\epsilon} \label{eqn:our_Regret}
\end{align}
where \[
R(S) := \sum_{t = 1}^{T} f(S) - f(S_t).
\]
This notion of regret captures the fact that if the algorithm plays a set in $\mc{S}^{k, \boldsymbol{\epsilon}}$ then they may be incurring up to $\mb{1}^T\boldsymbol{\epsilon}$ extra regret.
Note that when $\boldsymbol{\epsilon}=\mathbf{0}$ achieves the minimum (which can happen if the ``gaps'' between the greedily added element and all other elements at each stage is large) then this notion of regret is relative to the greedy set constructed in the noiseless setting.

% This notion of regret captures the fact that the approximate greedy set is the natural set to compare to and that this approximation inevitably results in additional regret.

The definition of regret in \eqref{eqn:our_Regret} is not novel to our paper. 
This notion is implicitly used in \cite{streeter_online_2007} in the proofs of Lemma 3 for the full-feedback setting and Theorem 13 for the bandit feedback setting, \cite{nie_explore-then-commit_nodate} in Theorem 4.1, \cite{pedramfar_stochastic_2023} in Theorem 1,  \cite{niazadeh_online_2023} in Theorem 2 for the full-feedback setting and Theorem 4 for bandit feedback, and \cite{10.5555/3618408.3619497} in Theorem 1.
% , both which generalize robust greedy solutions that retain the approximation ratio with additive error for a given optimization problem. 
% Thus, a more natural notion of regret for submodular functions incorporates the underlying computational difficulty. 
However, readers of these papers will note that they report their results not in terms of $\reg$, but $\alpha$-Regret: for an $\alpha \in [0, 1]$, define $\alpha$-regret by,
$
R_{\alpha} := \sum_{t = 1}^{T} \alpha f(S^*) - f(S_t) 
$
where $S^*:= \argmax_{|S| \le k} f(S)$. 
Using Lemma 1, one immediately has that $R_\alpha \leq \reg$ for $\alpha = (1 - e^{-1})$. 
Thus, an upper bound on \eqref{eqn:our_Regret} immediately results in an upper bound on $R_\alpha$, which is precisely what previous works exploit to obtain their upper bounds on $R_\alpha$. 

To summarize: all the analyses of these previous works concentrate on showing an upper bound on $\reg$, and only at the last step argue that $R_\alpha \leq \reg$, and report an upper bound on $R_\alpha$. 
But $R_\alpha$ can be a very loose lower bound on $\reg$!
For instance, when the function is modular (the inequalities of submodularity are tight), and the gap between the best set and worst set is equal to $\Delta < e^{-1}$, then a random selection algorithm would get zero or even negative $R_\alpha$ regret, while $\reg$ would be linear $\Delta T$, which is more natural. 
Thus, in studying regret against approximations attained by an offline step-wise greedy procedure, $R_{gr}$ can be a more appropriate measure than $R_{\alpha}$

% Deviating somewhat from previous works, instead of reporting our results in terms of $R_\alpha$, we choose to report our results in terms of the more critical quantity \eqref{eqn:our_Regret}, which can also be related to a lower bound.

%\lalit{Include a discussion of past definitions of regret and compare. }
%\lalit{Stronger in cases where $\alpha =0$, i.e. in lower bound where we greedy hits $S^{\star}$}

% \bibliography{ref2} % Use the table-specific .bib file
% \bibliographystyle{plain} % Use the table-specific style

\begin{table*}
\small
\centering
\scalebox{.84}{
\begin{tabular}{@{}lllll@{}}
\toprule
Function Assumptions & Stochastic & Regret & Upper Bound & Lower Bound \\ \midrule
Submodular+monotone & $\checkmark$ & $\reg$ &  \makecell[lt]{$kn^{1/3}T^{2/3}$ \\ \textsuperscript{\cite{pedramfar_stochastic_2023}}}
& \makecell[lt]{$\boldsymbol{\min_L(L^{1/3}n^{1/3}T^{2/3} + \sqrt{{n \choose k - L}T})}$ \\ \textsuperscript{\textbf{(This work)}}}
%$\sqrt{nT}$\textsuperscript{\cite{bandit_book}} 
\\
Submodular+monotone & $\times$ & $\reg$ &\makecell[lt]{ $kn^{1/3}T^{2/3}$ \\ \textsuperscript{\cite{streeter_online_2007}}}  &  \makecell[lt]{$\boldsymbol{\min_L(L^{1/3}n^{1/3}T^{2/3} + \sqrt{{n \choose k - L}T})}$ \\ \textsuperscript{\textbf{(This work)}}}%$\sqrt{nT}$\textsuperscript{\cite{bandit_book}} 
\\
Degree d Polynomial & $\times$ & $R(S^\ast)$ & \makecell[lt]{
 $\min(\sqrt{n^dT}, \sqrt{n^kT})$ \\ \textsuperscript{\cite{chen_adversarial_2021}}} & \makecell[lt]{$\min(\sqrt{n^dT}, \sqrt{n^kT})$ \\ \textsuperscript{\cite{chen_adversarial_2021}}}  \\
% Submodular+non-monotone & $\checkmark$ & $|S|\le k$ & $\mathcal{O}(n^{1/4}k^{5/4}T^{3/4})$  & $\Omega(n^{1/4}k^{5/4}T^{3/4})$  \\
% Submodular+non-monotone & $\times$ & $|S|\le k$ & $\mathcal{O}(n^{1/2}kT^{1/2})$  & $\Omega(n^{1/2}kT^{1/2})$  \\
\makecell[lt]{\textbf{Submodular+monotone} \\ \textbf{(This work)}} & $\checkmark$ & $\reg$ & $\boldsymbol{\min_L(Ln^{1/3}T^{2/3} + \sqrt{{n \choose k - L}T})}$  & $\boldsymbol{\min_i(L^{1/3}n^{1/3}T^{2/3} + \sqrt{{n \choose k - L}T})}$  \\ \bottomrule
\end{tabular}
}
\caption{Best known regret bounds for combinatorial multiarmed bandits under different assumptions. By lemma \ref{lemma:approx-greedy} our upperbound can also be stated for $R_{1 - e^{-1}}$.
We note that our lower bound proven for the stochastic setting immediately applies to the adversarial setting in the table.
% \kevin{Should there be a min over $\epsilon$ in the last row? Also, we should clarify in the caption that our upper bound is never worse than $k n^{1/3} T^{2/3}$ relative to $R_{1 - e^{-1}}$. }
}

\label{table:results}
\end{table*}

\subsection{Related Work}

There has been several works on combinatorial multi-armed bandits with submodular assumptions and different feedback assumptions.
Table~\ref{table:results} summarizes of the most relevant results as well as the results of this paper. For monotonic submodular maximization specifically, previous work use Lemma~\ref{lemma:approx-greedy} with appropriate $\boldsymbol{\epsilon}$ to prove an upper bound on expected $R_{\alpha}$-regret when the greedy result with perfect information gives an $\alpha$-approximation of the actual maximum value. 

% \lalit{What exactly are they guaranteeing on regret? What do they prove. }

\textbf{Stochastic}
In the stochastic setting, when the expected reward function is submodular and monotonic,
\cite{nie_explore-then-commit_nodate} proposed an explore-then-commit algorithm with full-bandit feedback that achieves $\reg = \mathcal{O}(k^{4/3}T^{2/3}n^{1/3})$\footnote{Most previous works, \cite{nie_explore-then-commit_nodate, pedramfar_stochastic_2023}, state their result in terms of $R_{\alpha}$ however, a careful analysis of the proofs of their main regret bounds show a stronger result in terms of $\reg$.}.  Recently, \cite{pedramfar_stochastic_2023} showed with the same explore-then-commit algorithm with different parameters, $\reg=\mathcal{O}(kn^{1/3}T^{2/3} + kn^{2/3}T^{1/3}d)$ is possible with delay feedback parameter of $d$. 
% For the pure-exploration setting, \cite{10.5555/3016100.3016183} proves a 
Without the monotonicity, \cite{pmlr-v206-fourati23a} achieves $R_{\alpha} = \mathcal{O}(nT^{2/3})$ with bandit feedback for $\alpha =1/2$. There have also been several works in the semi-bandit feedback setting (\cite{10.5555/3294996.3295062}, \cite{JMLR:v22:18-407}), and others such as getting the marginal gain of each element after each query. %All these upper-bounds can also be proven for $R_{\textbf{gr}}$ as they use a variant of $\epsilon$-greedy in their algorithm.
%$\min_{\epsilon} R(S^\epsilon_\text{gr}) + T \mb{1}^T\boldsymbol{\epsilon}$ as they use a variant of $\epsilon$-greedy in their algorithm. 

%\lalit{I am not sure what $O(\cdot)1-e^{-1}$ regret means. Do you mean $R_{1-e^{-1}} = O(\cdot)$?}

\textbf{Adversarial}
In the adversarial setting, the environment chooses an arbitrary sequence of monotone submodular functions $\{f_1, \ldots, f_T\}$, and the goal is to minimize regret against an approximation of the reward of the best set in hindsight (\cite{golovin2014onlinesubmodularmaximizationmatroid}, \cite{harvey_improved_2020}, \cite{NIPS2009_e0c64119}, \cite{pmlr-v202-wan23e}).
\cite{streeter_online_2008} showed $\mathcal{O}(k\sqrt{Tn\log n})$ $R_{(1 - e^{-1})}$-regret is possible with partially transparent feedback(where after each round, $f(S^{(i)})$ for all $i$ is revealed instead of only $f(S^{(k)})$) and $\mathcal{O}(kn^{1/3}T^{2/3})$ $R_{(1 - e^{-1})}$-regret for the bandit-feedback setting.
\cite{niazadeh_online_2023} proposed a generalized algorithm with $\mc{\tilde{O}}(kn^{2/3}T^{2/3})$ $R_{(1 - e^{-1})}$-regret with full bandit feedback, and showed all explore-then-commit greedy algorithms have $\Omega(T^{2/3})$ regret, when applied to our setting. Without the monotone assumption, \cite{niazadeh_online_2023} gets $\mathcal{O}(nT^{2/3})$ $R_{(1/2)}$-regret with bandit feedback. The upper-bound results in the adversarial setting doesn't naturally lead to results in the stochastic setting as the function is submodular and monotone only in expectation in the stochastic setting.  

\textbf{Continuous Submodular}
There are several works on online maximization of the continuous extensions of submodular set functions to a compact subspace such as Lovász and multilinear extensions(\cite{10.1007/s10107-018-1248-6}, \cite{NEURIPS2020_0f34132b}).
With a stronger assumption of DR-submodularity, it's possible to achieve higher approximation ratio guarantees and lower regret bounds (\cite{10.5555/3294771.3294818}, \cite{pmlr-v54-bian17a}, \cite{sadeghi_improved_2021}). \cite{pmlr-v202-wan23e} uses multilinear extension to achieve $O(T^{2/3})$ $R_{(1 - e^{-1})}$-regret for adversarial submodular maximization with partition matroid constraint. 

\textbf{Low-degree polynomial}
In general reward functions without the submodular assumption, \cite{chen_adversarial_2021} showed if the reward function is a $d$-degree polynomial, $\Theta\big(\min(\sqrt{n^dT}, \sqrt{n^kT})\big)$ regret is optimal.

\section{LOWER BOUND}

\begin{theorem}
\label{thm:main}
    For any $n \geq 4$, $k \le \lfloor n/3 \rfloor $, satisfying $512k^7n \le T \in \mathbb{N}$, 
    let $\mc{F}$ denote the set of submodular functions that are non-decreasing and bounded by $[0, 1]$ for sets of size $k$ or less, with $f(\emptyset) = 0$.  
    Then
    \begin{align*}
        \inf_{{\sf Alg}} \sup_{f \in \mc{F}}\E[ \reg ] \geq 
        \frac{1}{16} &(k - i^*)^{1/3}T^{2/3}n^{1/3} e^{-8 } + \frac{1}{4} T^{1/2}\sqrt{n - k \choose i^*} e^{-2}
    \end{align*}
    where the infimum is over all randomized algorithms and the supremum is over the functions in $\mc{F}$, and $i^* \in [k]$ is the largest value satisfying $\frac{16}{n^2k^6}{n - k \choose i^*}^3\le T$.
\end{theorem}

% \begin{figure}
%     \centering
%     \includegraphics[width=0.9\linewidth]{H0.png}
%     \caption{$\mc{H}_0$ expected reward plot for general $k$. The expected reward of $\mc{H}_0$ of sets $\{1, \ldots, i\}$ are elevated by $\Delta/k$ for $i < k$, and expected reward of set $\{1, \ldots, k\}$ is elevated by $\Delta$.}
%     \label{fig:hard-instance}
% \end{figure}

The lowerbound is intuitively a mix of the greedy explore-then-commit algorithm for the first $k - i^*$ arms, and then a standard MAB algorithm between all superarms of cardinality $k$ that include those elements. For small $T$ (i.e. $T = \mc{O}(n^4)$) the regret would be $\Omega(k^{1/3}n^{1/3}T^{2/3})$, and for large $T$(i.e. $T = \Omega(n^{3k - 2})$) the regret would be $\Omega({n \choose k}^{1/2}T^{1/2})$. This lowerbound also immediately gives a lower bound for the adversarial setting where $f_i = f + \mc{N}(0, 1)$ is the function chosen by the environment at time $i$.

% \kevin{Can we say anything when $T = O(k^3 n)$? }

% \kevin{Add text giving cases for $T$ where $i^*=1$ or $i^*=k$}

%\kevin{Sketch lower bound idea for $k=2$ to show where $T^{2/3}$ comes from, then move rest to the appendix}

% \begin{tikzpicture}
%     % Axes
%     \draw[->] (0,0) -- (6,0) node[right] {$i$};
%     \draw[->] (0,0) -- (0,2) node[above] {$f(i)$};
    
%     % Function plot
%     \foreach \i in {1,2,3,4,5}
%     {
%         \draw (\i, {1 - 1/\i}) -- (\i+1, {1 - 1/(\i+1)});
%     }

%     \foreach \i in {1,2,3,4,5}
%     {
%         \draw (\i, {1.1 - 1/\i} ) -- (\i+1, {1.1 - 1/(\i+1)});
%     }
    
%     % Dots
%     \filldraw[red] (1,0) circle (2pt);
%     \filldraw[red] (2,0.5) circle (2pt);
%     \filldraw[red] (3,0.6667) circle (2pt);
%     \filldraw[red] (4,0.75) circle (2pt);
%     \filldraw[red] (5,0.8) circle (2pt);
    
%     % Labels
%     \node[below] at (1,0) {$1$};
%     \node[below] at (2,0) {$2$};
%     \node[below] at (3,0) {$3$};
%     \node[below] at (4,0) {$4$};
%     \node[below] at (5,0) {$5$};
    
%     \node[left] at (0,.5) {$0.5$};
%     \node[left] at (0,1) {$0.67$};
%     \node[left] at (0,1.5) {$0.75$};
%     \node[left] at (0,2) {$0.8$};
    
% \end{tikzpicture}

\begin{sketch}
    We construct a hard instance so that at each cardinality a single set gives an elevated reward. Focusing on $k = 2$ for illustration, the instance would be the following:

    \begin{align*}
    \mb{H}_0 &:= \begin{cases} f(\{i\}) = 1/2 & \text{ if } i \in \{1\} \\
    f(\{i\}) = 1/2-\Delta & \text{ if } i \in [n] \setminus \{1\} \\
    f(\{i,j\}) = 3/4 & \text{ if } (i,j) = (1,2) \\
    % \mu_{i,j} = 3/4 - 2 \Delta & \text{ if } (i,j) = (\widehat{i},\widehat{j}) \\
    f(\{i,j\}) = 3/4 - \Delta & \text{ if } (i,j) \in \binom{[n]}{2} \setminus \{ (1,2)\} 
    \end{cases}
    \end{align*}
    where $\Delta$ is the gap of the best set that we will tune based on $T$. Pulling any arm of cardinality less than $2$ would incur $\Omega(1)$ regret, however, since there are only $n$ such sets (compared to ${n \choose 2}$ sets of size $2$), pulling these simple arms give more information on the optimal set.
    
    For a set of alternative instances, we choose a set of size $k$ and elevate its reward by $2\Delta$. We also elevate every prefix set of a permutation of this set by $2\Delta$ so that the new set can be found by a greedy algorithm. Again, for $k = 2$, and any  $\{\hat{i}, \hat{j}\} \in [n] \backslash \{1, 2\}$
        \begin{align*}
    \mb{H}_{\widehat{i},\widehat{j}} &:= \begin{cases} f(\{i\}) = 1/2 & \text{ if } i \in \{1\} \\
    f(\{i\}) = 1/2 + \Delta & \text{ if } i \in \{\widehat{i}\}
    %\tag{\lalit{If lifting by $\Delta/2$ should $\mu_{\hat{i}}=1$?} \artin{yes it should be lifted by $\Delta$}} 
    \\
    f(\{i\}) = 1/2-\Delta & \text{ if } i \in [n] \setminus \{1, \widehat{i}\} \\
    f(\{i,j\}) = 3/4 & \text{ if } (i,j) = (1,2) \\
    f(\{i,j\}) = 3/4 + \Delta & \text{ if } (i,j) = (\widehat{i},\widehat{j}) \\
    f(\{i,j\}) = 3/4 - \Delta & \text{ Otherwise } 
    \end{cases}
\end{align*}
    Note that, if $\Delta < \frac{1}{16}$ for the $k = 2$ instance, All the functions are submodular, as $f(\{a, b\}) - f(\{b\}) \le \frac{1}{4} + 2\Delta \le 1/2 - \Delta \le f(\{a\}) - f(\{\phi\})$ for any $a, b \in [n]$. 

For $\mb{H}_0$, if $\epsilon_i < \Delta$ for all $i \in [2]$, then $f_{\mc{H}_0}(S^{2,\boldsymbol{\epsilon}}_{gr}) = \frac{3}{4}$ as the noisy greedy finds the best arm, and otherwise $\mathbf{1}^T\boldsymbol{\epsilon} \ge \Delta$, so $\min_{\boldsymbol{\epsilon} \ge \mathbf{0}} f_{\mb{H}_0}(S^{2,\boldsymbol{\epsilon}}_{gr}) + \mathbf{1}^T\boldsymbol{\epsilon} = \frac{3}{4}$. Similarly, $\min_{\boldsymbol{\epsilon} \ge \mathbf{0}} f_{\mb{H}_{\widehat{i},\widehat{j}}}(S^{2,\boldsymbol{\epsilon}}_{gr}) + \mathbf{1}^T\boldsymbol{\epsilon} = \frac{3}{4} + \Delta$. So for these instances $\reg = R(S^*)$. 

We show that if the KL divergence between an alternate instance and $\mb{H}_0$ is small, then the algorithm cannot distinguish between the two environments and the maximum regret of the two would be $\Omega(\Delta T)$. 
Let $\P_{\widehat{i},\widehat{j}},\E_{\widehat{i},\widehat{j}}$ be the probability and expectation under $\mb{H}_{\widehat{i},\widehat{j}}$, respectively when executing some fixed algorithm with observations being corrupted by standard Gaussian noise.
Then
    $KL( \P_0 | \P_{\widehat{i},\widehat{j}} ) =  \frac{\Delta^2}{2} \big( \E_0[ T_{\widehat{i}} ] + 4\E_0[ T_{\widehat{i},\widehat{j}}] \big)$ for $k = 2$, where $T_S$ is the number of pulls of set $S$, and

\begin{align*}
    &\E_0[\reg] + \E_{\widehat{i},\widehat{j}}[\reg] \geq \tfrac{1}{2} \sum_{i=1}^n \E_0[ T_i ]+ \tfrac{\Delta T}{2} \Big( \P_0( T_{1,2} \leq \tfrac{T}{2} ) + \P_{\widehat{i},\widehat{j}}( T_{1,2} > \tfrac{T}{2} ) \Big) \\ &\geq \tfrac{1}{2}  \sum_{i=1}^n \E_0[ T_i ] + \tfrac{\Delta T}{4} \exp( - KL( \P_0 | \P_{\widehat{i},\widehat{j}} ) ) = \tfrac{1}{2}  \sum_{i=1}^n \E_0[ T_i ] + \tfrac{\Delta T}{4 } \exp\Big( - 2\Delta^2 \big( \E_0[ T_{\widehat{i}} ] + \E_0[ T_{\widehat{i},\widehat{j}}] \big) \Big) .
    % &\geq \min_{\lambda \in \triangle_n} \min_{\tau \leq T} \frac{1}{3} \tau  + \frac{\Delta T}{4 e} \exp( - 4 \Delta^2 \tau \max\{ \lambda_{\widehat{i}},\lambda_wpage{\widehat{j}} \} ) \\
    % &\geq \min_{\tau \leq T} \frac{1}{3} \tau  + \frac{\Delta T}{4 e} \exp( - 4 \Delta^2 \tau /(n-1) ) 
\end{align*}
Since $\widehat{i}, \widehat{j}$ were arbitrary, the following Lemma shows that there exist a pair that are pulled for small number of times in expectation (see Lemma \ref{lemma:pigeon} for general $k$).
\begin{lemma}
    There exists a pair $\widehat{i}, \widehat{j}$ such that \[\E_0[ T_{\widehat{i}} ] + \E_0[ T_{\widehat{i},\widehat{j}}] \le \frac{2 \sum_i \E_0[T_i]}{n - 2} + \frac{T}{{n - 2 \choose 2}}\]
\end{lemma}

\begin{proof}
    For a pair $(i, j)$, define $Q_{(i, j)} := \E_0[ T_{i} ] + \E_0[ T_{i,j}]$. Then the sum of this term for all pairs not equal to $1, 2$ would be
    \begin{align*}        
    Q &:= \sum_{(i, j) \ne (1, 2)} Q_{(i, j)} \le (n - 3) \sum_{i \ne (1, 2))} \E_0[T_{i}] + \sum_{i, j \ne 1, 2} \E_0[T_{i, j}] \le  (n - 3) \sum_{i} \E_0[T_{i}] + T 
    \end{align*}
    Then by Pigeonhole principal there exist a pair $\widehat{i}, \widehat{j}$ such that
    \begin{align*}        
    Q_{(\widehat{i}, \widehat{j})} \le \frac{Q}{{n - 2 \choose 2}} \le \frac{2}{n - 2} \sum_i \E_0[T_i] + \frac{T}{{n - 2 \choose 2}}  
    \end{align*}
\end{proof}
Using the lemma, for some $(\widehat{i},\widehat{j})$, we have
\begin{align*}
    &\E_0[\reg] + \E_{\widehat{i},\widehat{j}}[\reg]  \ge  \frac{1}{2}  \sum_{i=1}^n \E_0[ T_i ] + \tfrac{\Delta T}{4 } \exp\Big( - 2\Delta^2 \big( \frac{2}{n - 2} \sum_i \E_0[T_i] + \frac{T}{{n - 2 \choose 2}} \big) \Big)
\end{align*}
% Since $\widehat{i}, \widehat{j}$ were arbitrary, we can show that there exist such pair that $\E_0[ T_{\widehat{i}} ] + 4\E_0[ T_{\widehat{i},\widehat{j}}] \le \frac{\sum_i \E_0[T_i]}{n - 2} + \frac{4T}{{n - 2 \choose 2}}$ (see Lemma \ref{lemma:pigeon} for general $k$).
We choose an appropriate $\Delta$ based on value of $i^*$.
\begin{itemize}
    \item If $i^* = 1$, then for $\Delta = T^{-1/3}n^{1/3}$, we have $\frac{2\Delta^2 T}{{n - 2 \choose 2}} \le 1$. 
So either the KL divergence is less than $2$, then the regret is lowerbounded by $\Delta T e^{-2} = T^{2/3}n^{1/3} e^{-2}$ , or for KL divergence to be larger than $2$ we would have $\sum_i \E_0[T_i] \ge \frac{1}{4}T^{2/3}n^{1/3}$, which from the above equation shows the regret is $\Omega(T^{2/3}n^{1/3})$. 
This can be extended for expected value of pulls of each cardinality lower than $i^* + 1$ for general $k$. 

\item If $i^* = 2$, then it can be shown that the term $\frac{1}{2}  \sum_{i=1}^n \E_0[ T_i ] + \frac{\Delta T}{4 } \exp\Big( - 2\Delta^2 \big( \frac{2}{n - 2}\sum_i \E_0[ T_{i} ] + (T - \sum_{i=1}^n \E_0[ T_i ])/{n - 2 \choose 2} \big) \Big)$ with $\Delta = \sqrt{{n - 2 \choose 2} / T}$ minimizes when $\sum_{i = 1}^{n} \E_0[ T_{i}] = 0$ i.e. zero single arm sets being pulled in expectation, so the regret would be $T^{1/2}{n - 2 \choose 2}^{1/2} \exp(-2)$.
\end{itemize}
This shows that the expected regret is $\tilde{\Omega}(\min_i(i^{1/3}n^{1/3}T^{2/3} + \sqrt{{n \choose k - i}T}))$. The instance of general $k$, and the detailed proof is in appendix \ref{proof:lower}.
\end{sketch}

We define an algorithm to be in non-adaptive greedy error-threshold class against $\reg$ regret, if it selects $\epsilon'_1, \ldots, \epsilon'_k$ at the start only dependent on parameters $T, n, k$ before any arm pulls, and minimizes regret against $f(S^{k, \boldsymbol{\epsilon'}}_\textbf{gr}) + \mb{1}^T\boldsymbol{\epsilon'}$. All the algorithms from previous work in the literature fall within this restricted class, and with this extra assumption we can prove a stronger lower bound. 

\begin{theorem}
\label{thm:naet-lower}
    For any $n \geq 4$, $k \le \lfloor n/3 \rfloor $, satisfying $512k^9n \le T \in \mathbb{N}$, 
    let $\mc{F}$ denote the set of submodular functions that are non-decreasing and bounded by $[0, 1]$ for sets of size $k$ or less, with $f(\emptyset) = 0$.  
    Then
    \begin{align*}
        \inf_{{\sf Alg \in NAET}} \sup_{f \in \mc{F}}\E[ \reg ] \geq 
        \frac{1}{288} &(k - i^*)T^{2/3}n^{1/3} e^{-10 } + \frac{1}{4} T^{1/2}\sqrt{n - k \choose i^*} e^{-2}
    \end{align*}
    where the infimum is over all randomized algorithms with non-adaptive greedy error threshold selection, and the supremum is over the functions in $\mc{F}$, and $i^* \in [k]$ is the largest value satisfying $\frac{16}{n^2k^6}{n - k \choose i^*}^3\le T$.
\end{theorem}

\section{UCB UPPER BOUND}
% \begin{algorithm}[H]
% \caption{Sub-UCB algorithm for set bandits with cardinality constraints}
% \begin{algorithmic}[1]
% \State \textbf{Input:} $m$, $T$
% \State \textbf{Initialization:} $S^{(0)} = \phi$, $T_A=0$ and $U_a(0)=0$ for all simple arms $a \in [n]$
% \State Pull each arm $m$ times and update $t \leftarrow mn, T_{\{a\}} \leftarrow m \quad \forall a \in [m]$
% \For{$i=1,2,\dots,k$}
%     \While{$T_{S^{(i-1)} \cup \argmax U_a(t)} < m$
%         \For{each $a \notin S^{(i - 1)}$}
%             \State $\hat{\mu}_{S^{(i-1)}\cup\{a\}}(\alpha) \leftarrow \frac{1}{T_{S^{(i-1)}\cup\{a\}} + \alpha(T_{S^{(i-1)}} + T_a)}\Big[ 
%  \sum_{t: I_t = {S^{(i-1)}} \cup \{a\}} r_t +  \alpha \big( \frac{T_a + T_{S^{(i-1)}}}{T_{S^{(i-1)}}} \sum_{t: I_t = {S^{(i-1)}}} r_t + \frac{T_a + T_{S^{(i-1)}}}{T_a} \sum_{t: I_t = a} r_t \big) \Big]$
%             \State Compute UCB for arm $S^{(i-1)} \cup \{a\}$: $U_a(t)= \min_{\alpha \le 1} \hat{\mu}_{S^{(i-1)}\cup\{a\}}(\alpha) + \sqrt{\frac{8 \log t}{T_{S^{(i-1)} \cup \{a\}} + \alpha(T_{a} + T_{S^{(i-1)}})}}$
%         \EndFor
%         \State Choose arm $S^{(i-1)} \cup \arg\max_{a} U_a(t)$ and observe reward $r_t$.
%         \State $T_{S^{(i-1)} \cup \{a\}} \leftarrow T_{S^{(i-1)} \cup \{a\}} + 1$
%     \EndWhile
%     \State Update the base set: $S^{(i)} \leftarrow S^{(i-1)} \cup \{a_i\}$ where $a_i := \arg\max_{a} U_a(t)$
% \EndFor
% \While{$t < T$}
%         \State Choose arm $S^{(k)}$ and observe reward $r_t$ and update $t \leftarrow t + 1$}
% \EndWhile
% \end{algorithmic}
% \end{algorithm}

\begin{algorithm}
\small
\caption{\subucb algorithm for set bandits with cardinality constraints}
\begin{algorithmic}[1]
\State \textbf{Input:} $T$, $m$, greedy stop level $l$
\State \textbf{Initialization:} $S^{(0)} = \emptyset$, $T_A=0$ for all $A \subset [n]$ %\lalit{Should $T_A = 0, \forall A\in 2^{[n]}$}
% \State $i^* \leftarrow \min\{k, \lfloor \frac{1}{3} (\log_{n} T + 2) \rfloor \}$ %\lalit{Same $i^{\ast}$ as above?} \artin{Yes}
\State For each $a\in [n]$, pull $\{a\}$ exactly $m$ times and update $T_{\{a\}}\gets m$. Update $t \leftarrow mn$.
\For{$i=1,2,\dots,l$}
    \State $U_a \gets \infty$ for all $a\not \in S^{(i-1)}$
    \While{$T_{S^{(i-1)} \cup \argmax U_a} < m$}
        \State Pull arm $S_t = S^{(i-1)} \cup \arg\max_{a} U_a$, observe $r_t$, and update $T_{S_t} \leftarrow T_{S_t} + 1$
        \For{each $a \notin S^{(i - 1)}$}
            \State $S_a \gets S^{(i-1)} \cup \{a\}$
            \State $\hat{\mu}_{S_a} \gets \frac{1}{T_{S_a}} 
 \sum_{t: I_t = S_a} r_t$

            %\STATE $\hat{\mu}_{S^{(i-1)}\cup\{a\}} \gets \frac{1}{T_{S^{(i-1)}\cup\{a\}}}  \sum_{t: I_t = {S^{(i-1)}} \cup \{a\}} r_t$
            \State Compute UCB: $U_a= \hat{\mu}_{S_a} + \sqrt{\frac{8 \log t}{T_{S_a} }}$
            %\STATE Compute UCB for arm $S^{(i-1)} \cup \{a\}$ $U_a(t)= \hat{\mu}_{S^{(i-1)}\cup\{a\}} + \sqrt{\frac{8 \log t}{T_{S^{(i-1)} \cup \{a\}} }}$
        \EndFor
        \State $t\gets t+1$
    \EndWhile
    \State Update the base set: $S^{(i)} \leftarrow S^{(i-1)} \cup \{a_i\}$ where $a_i := \arg\max_{a} U_a$
\EndFor
\While{$t < T$}
        \State Run UCB on all size $k$ super-arms $A$ where $S^{(l)} \in A$.
\EndWhile
\end{algorithmic}
\end{algorithm}
% \lalit{TODO: Add text about what the algorithm is doing. }

A natural approach to minimizing regret is to take an Explore-Then-Commit strategy motivated by the greedy algorithm. Such an algorithm would be the following - proceed in $k$ rounds. Set $S^{0} = \emptyset$. 
In round $i$ pull each set in the collection $\{S^{i-1}\cup \{a\}:a\in [n]\setminus S^{i-1}\}$, $m$ times. 
Use these samples to update our estimate $\hat{f}$ of $f$ on these sets, and set $S^{(i)}\leftarrow \arg\max_{a\in [n]\setminus S^{i-1}} \hat{f}(S^{i-1}\cup \{a\})$. 
This approach has been pursued by existing works \cite{nie_explore-then-commit_nodate}, and with an appropriate choice of $m$ results in $O(kn^{1/3}T^{2/3})$ regret.

The disadvantage of this approach is that it can not achieve the correct trade-off between $\sqrt{n^k T}$ and $kn^{1/3}T^{2/3}$ exhibited by the lower bound. Motivated by the statement of the lower bound, our algorithm \subucb attempts to interpolate between these different regret regimes. The critical quantity is $i^{\ast}$. For the first $k-i^{\ast}$ cardinalities, our algorithm plays a UCB style strategy which more or less follows the ETC strategy described in the previous paragraph. After that, it defaults to a UCB algorithm on all subsets containing $S^{k-i^{\ast}}$, a total of ${n - k + i^* \choose i^*}$ possible arms.

%\subucb is our proposed algorithm. A critical quantity in this algorithm is $i^{\ast}$ motivated by the same quantity in the lower bound Theorem~\ref{thm:main}. For the first $k - i^*$ cardinalities starting with null set, \subucb explores all single-arm additions to the base set. 

%until the arm addition set with maximum UCB has been sampled $m$ times, where at that time it updates the base set to that set (increasing its cardinality by one), and repeating the procedure. After the base set has $k - i^*$ arms, the algorithm runs UCB for all super arms of cardinality $k$ containing the base set (for which there are ${n - k + t^* \choose i^*}$ possible sets).

% \lalit{Can you use the same $i^{\ast}$ from the lower bound?}

%As we go to higher cardinalities, samples of the previous levels can be used to get smaller upper confidence bounds because of the submodular inequality $f(S \cup \{a\}) \le f(S) + f(\{a\})$.
%\kevin{But this algorithm doesn't exploit this right?}

    % Let $c$ be the submodular curvature of the mean reward function $f$, then $f(S) + f(a) - f(S \cup \{a\}) \le cf(a) \le c$, and
    % \begin{align*}
    % E[\hat{\mu}_{S \cup \{a\}}] &= \frac{1}{T_{S\cup\{a\}} + \alpha(T_S + T_a)} \Big[T_{S \cup \{a\}} f(S \cup \{a\}) + \alpha(T_a + T_S) \Big] \\
    % &\le f(S \cup\{a\}) + \frac{\alpha(T_a + T_S) c}{T_{S\cup\{a\}} + \alpha(T_S + T_a)} 
    % \end{align*}
    %\lalit{Be careful about the notion of Regret. }
    \begin{theorem}
        \label{thm:regret}
        For any $l \le k$, \subucb  guarantees 
        \begin{align*}
            \mathbb{E}[\reg] &\leq (1 + 4\sqrt{2})lT^{2/3}n^{1/3}(\log T)^{1/3} + 65\sqrt{T{n - k \choose k - l} \log T} + \frac{32}{15}{n - k \choose k - l}
        \end{align*}
        when $m=T^{2/3}n^{-2/3}\log{T}^{1/3}$.  
        %\subucb achieves expected regret $\min_{0 \le i \le k} (1 + 4\sqrt{2})iT^{2/3}n^{1/3}(\log T)^{1/3} + \sqrt{T{n - k \choose k - i}}$ relative to $\frac{1}{c}(1 - e^{-c})$ approximation of the maximum value of size $k$ when $m = T^{2/3}n^{-2/3}\log{T}^{1/3}$.
    \end{theorem}
    % \todol{Can you detail the subroutine you are using to get the $1/c(1-e^{-c})$ guarantee? Maybe in a new algorithm environment.}
    % \todol{Please tell me how to choose $m$ in the statemetn of the theorem?}
    \begin{sketch}
        We show that for $\epsilon := 2\sqrt{2\log(2knT^2)/m}$, the greedy part of \subucb with high probability adds an $\epsilon$-optimal arm in each step.
        Defining event $G$ to be $|\hat{\mu}_S - f(S)| \le \sqrt{2T_{S}\log(2knT^2)}$ for all iterations, we prove that this event is true with a probability of at least $1 - \frac{1}{T}$.

        On Event $G$, We show that an $\epsilon$-good arm is selected at each step of the greedy algorithm for $\epsilon = 2\sqrt{\frac{2 \log (2knT^2)}{m}}$. 
        Let $a$ be a sub-optimal arm with expected reward value more than $2\sqrt{\frac{2 \log (2knT^2)}{m}}$ from the best arm  in the $i$-th step i.e. $\Delta_{S^{(i)},a} := \max_{a'} f(S^{(i)} \cup \{a'\}) - f(S^{(i)} \cup \{a\}) \ge 2\sqrt{\frac{2 \log (2knT^2)}{m}}$. Then if arm $a$ is added in $i$-th step, we have $U_a(t) \ge U_{a^*}(t) \ge f(S^{(i)) \cup \{a^*\}}$, and therefore,
    \begin{align*}
        &U_a(t) - f(S^{(i)} \cup \{a\}) \ge \Delta_{S^{(i)},a} > 2\sqrt{\frac{2 \log (2knT^2)}{m}},
    \end{align*}
    so $\hat{\mu}_{S^{(i)} \cup \{a\}}- f(S^{(i)} \cup \{a\}) > \sqrt{\frac{2 \log (2knT^2)}{m}}$.
    This is a contradiction with event $G$, so on event $G$ such an arm cannot be selected. Lastly, we expand the regret of two stages. As UCB in the second part of the algorithm has the regret of $65\sqrt{T{n - k \choose k - l} \log T} + \frac{32}{15}{n - k \choose k - l}$ against the best arm containing $S^{(l)}$(see \cite{bandit_book}), it is an upper bound for the regret against the greedy solution were the first $l$ steps select an $\epsilon$-good arm, and the last $k - l$ steps select the best arm, so on event $G$ the regret can be written against a set in $\mc{S}^{k, \boldsymbol{\epsilon}}$ where \[\mathbf{1}^T\boldsymbol{\epsilon} = l\epsilon + (k - l)0 = 2l\sqrt{\frac{2 \log (2knT^2)}{m}}.\]
    Therefore, the expected regret $\E[\reg]$ on event $G$ can be written as
    \begin{align*}        
    &2Tl\sqrt{\frac{2  \log (2knT^2)}{m}} + mn(k - i^*) +65\sqrt{T{n - k \choose k - l} \log T} + \frac{32}{15}{n - k \choose k - l},
    \end{align*}
    for any choice of $m$ and $l$. So for $m = T^{2/3}n^{-2/3}\log^{1/3}(2knT^2)$ the above term becomes $\mc{\tilde{O}}(lT^{2/3}n^{1/3} + \sqrt{T{n \choose k - l}})$. The detailed proof is in Appendix \ref{proof:upper}
    \end{sketch}
 %\lalit{What is the coupling between $T, m, \epsilon$? We should say that $S^{k-i^*}\in \mc{S}^{\epsilon}$ for the right choice of $\epsilon$, so the result follows. }

%\artin{TODO: mention that applying this algorithm on the hard instance of lowerbound, the regret bound is against S^*}

\begin{figure}
    \centering
    \includegraphics[width=0.6\linewidth]{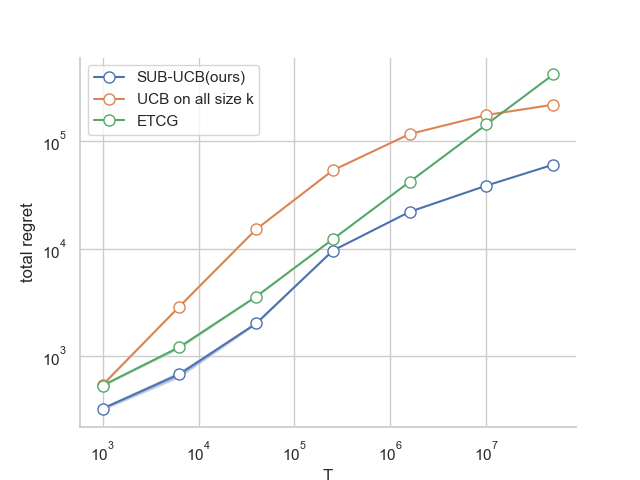}
    \caption{Regret comparison for weighted set cover with $n=15$ and $k = 4$}
    \label{fig:weighted-cover}
\end{figure}

\section{EXPERIMENTS}

For the experiments we compare \subucb($l$) for different greedy stop levels $l$, \subucb($k - i^*$) which selects the best stop level based on the regret analysis, the ETCG (explore-then-commit greedy) algorithm from \cite{nie_explore-then-commit_nodate}, and UCB on all size $k$ arms. Each arm pull has a $1$-Gaussian noise, with $50$ trials for each setting. The expected reward functions are the following. 

\newpage
\textbf{Functions:}
\begin{itemize}
    %\item Modular function i.e. $f(S) = \sum_{s \in S} f(s)$, with $n = 10$ and $k = 3$, and for simple arm means, we assigned $0.9$ mean to one of the arms randomly, and $0.1$ mean to the rest for the modular function.
    \item The Unique greedy path hard instance i.e. $$f(S) = \begin{cases}
        \sum_{i = 1}^{|S|} \frac{1}{k + i} \quad S = \{1, \ldots, |S|\} \\
        \sum_{i = 1}^{|S|} \frac{1}{k + i} + \frac{1}{100} \quad S = \{1, \ldots, |S|\}.
    \end{cases}$$
    This function is inspired by the hard instance in the proof of our lower-bound. Note that this particular parameterization is submodular when $k \le 7$, not for general $k$.
    \item Weighted set cover function i.e. $f_\mc{C}(S) = \sum_{C \in \mc{C}} w(C) \mathbf{1}\{S \cap C \ne \emptyset\}$ for a partition $\mc{C}$ of $[n]$ and weight function $w$ on the partition. For $n = 15$ and $k=4$, we use the partitions of size $5,5,4,1$ with weights of $1/10,1/10,2/10,6/10$ respectively.
\end{itemize}

% \begin{figure}[h]
%     \centering
%     \includegraphics[width=.5\linewidth]{media/new-i-subucb.png}
%     \caption{Regret of $\subucb(i)$ for the unique greedy path reward function. The optimal stop greedy cardinality $l = k - i^*$ is highlighted}
%     \label{fig:level-comp}
% \end{figure}

\begin{figure*}[b]
    \centering
\begin{subfigure}[b]{.49\columnwidth}
    \includegraphics[width=\columnwidth]{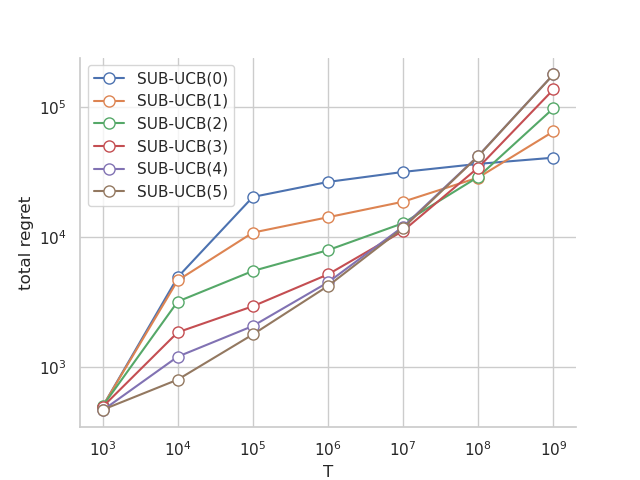}
\end{subfigure}
\hfill
\begin{subfigure}[b]{.49\columnwidth}
    \includegraphics[width=\columnwidth]{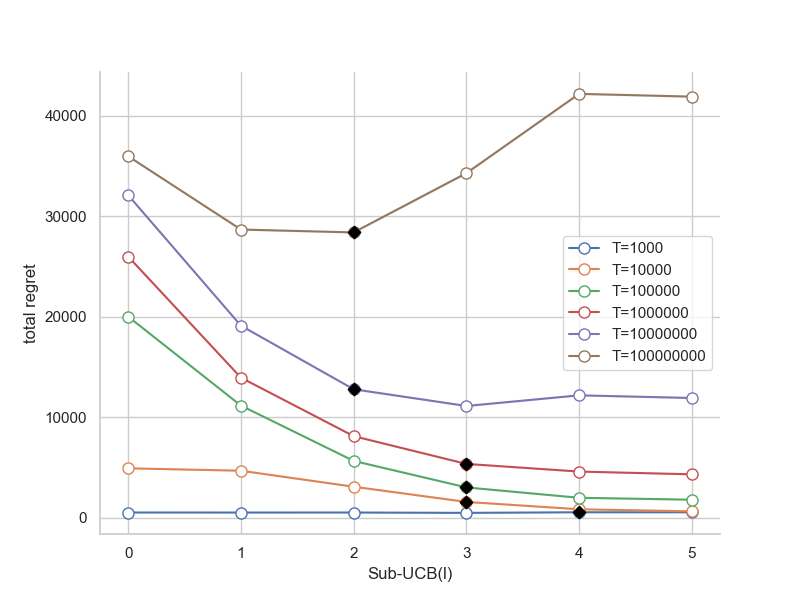}
\end{subfigure}
\caption{Comparison between all \subucb greedy stop cardinality choices for the unique greedy path function with $n = 20$ and $k = 5$. The worst-case optimal stop cardinality $l = k - i^*$ is highlighted}
\label{fig:hard-all}
\end{figure*}

\textbf{Results:} As illustrated in figure \ref{fig:weighted-cover}, we observe that our algorithm with the level selection of $k - i^*$ outperforms both ETCG and naive UCB on all size $k$ arms, as it combines the advantages of greedy approach for small $T$s and UCB on many super arms for large $T$. For smaller $T$s compared to ${n \choose k}$, both \subucb and ETCG outperform normal UCB as it doesn't have enough budget to find optimal sets of size $k$, so it gets linear regret(as the other two get $O(T^{2/3})$). However, as $T$ becomes larger the reverse happens as ${n \choose k}T^{1/2}$ becomes smaller than $T^{2/3}$, but \subucb adopts to $T$ and continues to outperform the two until it converges with naive UCB for very large $T$.

In figure \ref{fig:hard-all}, we compare the performance of \subucb for different choices of greedy stop cardinality, and observe that the best choice gradually decreases from $k$ to $0$ as $T$ gets larger, and $k - i^*$ is a practical selection of the best stop cardinality before running the algorithm. Note that the defined stop level was chosen to minimize the worst-case bound on the regret, and if the gaps between arms on a particular instance are larger than the worst case, this stop level could be conservative. So $k - i^*$ is near the optimal stop level, and not the exact one as seen in these figures. Also, the empirical standard derivation is much smaller than $\mathcal{O}(T^{1/2})$ due to the regret symmetry of non-optimal sets at each cardinality, and it's not visible in the plots.

\section{CONCLUSION}
In this paper we showed that $\min_L(L^{1/3}T^{2/3}n^{1/3} + \sqrt{{n \choose k - L}T})$, ignoring logarithmic factors, is a lower bound on the regret against robust greedy solutions of stochastic submodular functions, and a stronger lower bound if the algorithm class is slightly restricted. We also matched this bound with an algorithm.
This work is the first minimax lower bound for submodular bandits, and beyond closing the $k^{2/3}$ gap between the general lowerbound and upperbound, it remains open to prove similar minimax optimal bounds in settings with different types of constraint such as matroid, or in general, any offline-to-online greedy procedure that is robust to local noise (e.g. Non-monotonic submodular maximization where the greedy approach gets a $1/2$-approximation of the function, or DR-submodular optimization for the continuous setting which also has a $(1 - e^{-1})$-approximation).

\section*{Acknowledgements}
This paper is based on work supported by
Microsoft Grant for Customer Experience Innovation, NSF Award TRIPODS 2023239, and DMR-2308979.

\nocite{*}

\bibliographystyle{apalike}
\bibliography{refs}

\begin{thebibliography}{}

\bibitem[Agarwal et~al., 2020]{agarwal_dart_2020}
Agarwal, M., Aggarwal, V., Quinn, C.~J., and Umrawal, A. (2020).
\newblock {DART}: {aDaptive} {Accept} {RejecT} for non-linear top-{K} subset identification.
\newblock arXiv:2011.07687 [cs, stat].

\bibitem[Agarwal et~al., 2021]{agarwal_stochastic_2021}
Agarwal, M., Aggarwal, V., Quinn, C.~J., and Umrawal, A.~K. (2021).
\newblock Stochastic {Top}-\${K}\$ {Subset} {Bandits} with {Linear} {Space} and {Non}-{Linear} {Feedback}.
\newblock arXiv:1811.11925 [cs, stat].

\bibitem[Bach, 2019]{10.1007/s10107-018-1248-6}
Bach, F. (2019).
\newblock Submodular functions: from discrete to continuous domains.
\newblock {\em Math. Program.}, 175(1–2):419–459.

\bibitem[Balcan and Harvey, 2011]{balcan_learning_nodate}
Balcan, M.-F. and Harvey, N. J.~A. (2011).
\newblock Learning {Submodular} {Functions}.

\bibitem[Balkanski and Singer, 2018]{balkanski_adaptive_2018}
Balkanski, E. and Singer, Y. (2018).
\newblock The adaptive complexity of maximizing a submodular function.
\newblock In {\em Proceedings of the 50th {Annual} {ACM} {SIGACT} {Symposium} on {Theory} of {Computing}}, pages 1138--1151, Los Angeles CA USA. ACM.

\bibitem[Bian et~al., 2019]{bian_guarantees_2019}
Bian, A.~A., Buhmann, J.~M., Krause, A., and Tschiatschek, S. (2019).
\newblock Guarantees for {Greedy} {Maximization} of {Non}-submodular {Functions} with {Applications}.
\newblock arXiv:1703.02100 [cs, math].

\bibitem[Bian et~al., 2017a]{10.5555/3294771.3294818}
Bian, A.~A., Levy, K.~Y., Krause, A., and Buhmann, J.~M. (2017a).
\newblock Continuous dr-submodular maximization: structure and algorithms.
\newblock In {\em Proceedings of the 31st International Conference on Neural Information Processing Systems}, NIPS'17, page 486–496, Red Hook, NY, USA. Curran Associates Inc.

\bibitem[Bian et~al., 2017b]{pmlr-v54-bian17a}
Bian, A.~A., Mirzasoleiman, B., Buhmann, J., and Krause, A. (2017b).
\newblock {Guaranteed Non-convex Optimization: Submodular Maximization over Continuous Domains}.
\newblock In Singh, A. and Zhu, J., editors, {\em Proceedings of the 20th International Conference on Artificial Intelligence and Statistics}, volume~54 of {\em Proceedings of Machine Learning Research}, pages 111--120. PMLR.

\bibitem[Chen et~al., 2018]{pmlr-v84-chen18f}
Chen, L., Hassani, H., and Karbasi, A. (2018).
\newblock Online continuous submodular maximization.
\newblock In Storkey, A. and Perez-Cruz, F., editors, {\em Proceedings of the Twenty-First International Conference on Artificial Intelligence and Statistics}, volume~84 of {\em Proceedings of Machine Learning Research}, pages 1896--1905. PMLR.

\bibitem[Chen et~al., 2016a]{chen_combinatorial_2016-1}
Chen, W., Hu, W., Li, F., Li, J., Liu, Y., and Lu, P. (2016a).
\newblock Combinatorial {Multi}-{Armed} {Bandit} with {General} {Reward} {Functions}.
\newblock In {\em Advances in {Neural} {Information} {Processing} {Systems}}, volume~29. Curran Associates, Inc.

\bibitem[Chen et~al., 2016b]{chen_combinatorial_2016}
Chen, W., Wang, Y., Yuan, Y., and Wang, Q. (2016b).
\newblock Combinatorial {Multi}-{Armed} {Bandit} and {Its} {Extension} to {Probabilistically} {Triggered} {Arms}.
\newblock arXiv:1407.8339 [cs].

\bibitem[Chen et~al., 2021]{chen_adversarial_2021}
Chen, X., Han, Y., and Wang, Y. (2021).
\newblock Adversarial {Combinatorial} {Bandits} with {General} {Non}-linear {Reward} {Functions}.
\newblock arXiv:2101.01301 [cs, stat].

\bibitem[Feige, 1998]{10.1145/285055.285059}
Feige, U. (1998).
\newblock A threshold of ln n for approximating set cover.
\newblock {\em J. ACM}, 45(4):634–652.

\bibitem[Feldman and Karbasi, 2020]{NEURIPS2020_0f34132b}
Feldman, M. and Karbasi, A. (2020).
\newblock Continuous submodular maximization: Beyond dr-submodularity.
\newblock In Larochelle, H., Ranzato, M., Hadsell, R., Balcan, M., and Lin, H., editors, {\em Advances in Neural Information Processing Systems}, volume~33, pages 1404--1416. Curran Associates, Inc.

\bibitem[Fourati et~al., 2023]{pmlr-v206-fourati23a}
Fourati, F., Aggarwal, V., Quinn, C., and Alouini, M.-S. (2023).
\newblock Randomized greedy learning for non-monotone stochastic submodular maximization under full-bandit feedback.
\newblock In Ruiz, F., Dy, J., and van~de Meent, J.-W., editors, {\em Proceedings of The 26th International Conference on Artificial Intelligence and Statistics}, volume 206 of {\em Proceedings of Machine Learning Research}, pages 7455--7471. PMLR.

\bibitem[Goemans et~al., 2009]{goemans_approximating_2009}
Goemans, M.~X., Harvey, N. J.~A., Iwata, S., and Mirrokni, V. (2009).
\newblock Approximating {Submodular} {Functions} {Everywhere}.
\newblock In {\em Proceedings of the {Twentieth} {Annual} {ACM}-{SIAM} {Symposium} on {Discrete} {Algorithms}}, pages 535--544. Society for Industrial and Applied Mathematics.

\bibitem[Golovin et~al., 2014]{golovin2014onlinesubmodularmaximizationmatroid}
Golovin, D., Krause, A., and Streeter, M. (2014).
\newblock Online submodular maximization under a matroid constraint with application to learning assignments.

\bibitem[Hao et~al., 2021]{hao_high-dimensional_2021}
Hao, B., Lattimore, T., and Wang, M. (2021).
\newblock High-{Dimensional} {Sparse} {Linear} {Bandits}.
\newblock arXiv:2011.04020 [cs, math, stat].

\bibitem[Harvey et~al., 2020]{harvey_improved_2020}
Harvey, N., Liaw, C., and Soma, T. (2020).
\newblock Improved {Algorithms} for {Online} {Submodular} {Maximization} via {First}-order {Regret} {Bounds}.
\newblock In {\em Advances in {Neural} {Information} {Processing} {Systems}}, volume~33, pages 123--133. Curran Associates, Inc.

\bibitem[Hill et~al., 2017]{hill2017efficient}
Hill, D.~N., Nassif, H., Liu, Y., Iyer, A., and Vishwanathan, S. (2017).
\newblock An efficient bandit algorithm for realtime multivariate optimization.
\newblock In {\em Proceedings of the 23rd ACM SIGKDD International Conference on Knowledge Discovery and Data Mining}, pages 1813--1821.

\bibitem[Kaufmann et~al., 2016]{kaufmann_complexity_2016}
Kaufmann, E., Cappé, O., and Garivier, A. (2016).
\newblock On the {Complexity} of {Best} {Arm} {Identification} in {Multi}-{Armed} {Bandit} {Models}.
\newblock arXiv:1407.4443 [cs, stat].

\bibitem[Kearns and Singh, 2002]{kearns2002near}
Kearns, M. and Singh, S. (2002).
\newblock Near-optimal reinforcement learning in polynomial time.
\newblock {\em Machine learning}, 49(2):209--232.

\bibitem[Krause and Golovin, 2014]{bordeaux_submodular_2014}
Krause, A. and Golovin, D. (2014).
\newblock Submodular {Function} {Maximization}.
\newblock In Bordeaux, L., Hamadi, Y., and Kohli, P., editors, {\em Tractability}, pages 71--104. Cambridge University Press, 1 edition.

\bibitem[Lattimore and Szepesvari, 2017]{bandit_book}
Lattimore, T. and Szepesvari, C. (2017).
\newblock Bandit algorithms.

\bibitem[Matsuoka et~al., 2021]{matsuoka_tracking_2021}
Matsuoka, T., Ito, S., and Ohsaka, N. (2021).
\newblock Tracking {Regret} {Bounds} for {Online} {Submodular} {Optimization}.
\newblock In {\em Proceedings of {The} 24th {International} {Conference} on {Artificial} {Intelligence} and {Statistics}}, pages 3421--3429. PMLR.
\newblock ISSN: 2640-3498.

\bibitem[Nemhauser and Wolsey, 1978]{nemhauser1978submodular}
Nemhauser, G.~L. and Wolsey, L.~A. (1978).
\newblock Submodular set functions, matroids and the greedy algorithm: Tight worst-case bounds and some generalizations of the rado-edmonds theorem.
\newblock {\em Combinatorica}, 3(3-4):257--268.

\bibitem[Niazadeh et~al., 2023]{niazadeh_online_2023}
Niazadeh, R., Golrezaei, N., Wang, J., Susan, F., and Badanidiyuru, A. (2023).
\newblock Online {Learning} via {Offline} {Greedy} {Algorithms}: {Applications} in {Market} {Design} and {Optimization}.
\newblock arXiv:2102.11050 [cs, math, stat].

\bibitem[Nie et~al., 2022]{nie_explore-then-commit_nodate}
Nie, G., Agarwal, M., Umrawal, A.~K., Aggarwal, V., and Quinn, C.~J. (2022).
\newblock An {Explore}-then-{Commit} {Algorithm} for {Submodular} {Maximization} {Under} {Full}-bandit {Feedback}.

\bibitem[Nie et~al., 2023]{10.5555/3618408.3619497}
Nie, G., Nadew, Y.~Y., Zhu, Y., Aggarwal, V., and Quinn, C.~J. (2023).
\newblock A framework for adapting offline algorithms to solve combinatorial multi-armed bandit problems with bandit feedback.
\newblock In {\em Proceedings of the 40th International Conference on Machine Learning}, ICML'23. JMLR.org.

\bibitem[Pasteris et~al., 2024]{pmlr-v238-pasteris24a}
Pasteris, S.~U., Rumi, A., Vitale, F., and Cesa-Bianchi, N. (2024).
\newblock Sum-max submodular bandits.
\newblock In Dasgupta, S., Mandt, S., and Li, Y., editors, {\em Proceedings of The 27th International Conference on Artificial Intelligence and Statistics}, volume 238 of {\em Proceedings of Machine Learning Research}, pages 2323--2331. PMLR.

\bibitem[Pedramfar and Aggarwal, 2023]{pedramfar_stochastic_2023}
Pedramfar, M. and Aggarwal, V. (2023).
\newblock Stochastic {Submodular} {Bandits} with {Delayed} {Composite} {Anonymous} {Bandit} {Feedback}.
\newblock arXiv:2303.13604 [cs].

\bibitem[Roughgarden and Wang, 2018]{pmlr-v75-roughgarden18a}
Roughgarden, T. and Wang, J.~R. (2018).
\newblock An optimal learning algorithm for online unconstrained submodular maximization.
\newblock In Bubeck, S., Perchet, V., and Rigollet, P., editors, {\em Proceedings of the 31st Conference On Learning Theory}, volume~75 of {\em Proceedings of Machine Learning Research}, pages 1307--1325. PMLR.

\bibitem[Sadeghi et~al., 2021]{sadeghi_improved_2021}
Sadeghi, O., Raut, P., and Fazel, M. (2021).
\newblock Improved {Regret} {Bounds} for {Online} {Submodular} {Maximization}.
\newblock arXiv:2106.07836 [cs, math, stat].

\bibitem[Simchowitz et~al., 2016]{simchowitz_best--k_2016}
Simchowitz, M., Jamieson, K., and Recht, B. (2016).
\newblock Best-of-{K} {Bandits}.
\newblock arXiv:1603.02752 [cs, stat].

\bibitem[Singla et~al., 2016]{10.5555/3016100.3016183}
Singla, A., Tschiatschek, S., and Krause, A. (2016).
\newblock Noisy submodular maximization via adaptive sampling with applications to crowdsourced image collection summarization.
\newblock In {\em Proceedings of the Thirtieth AAAI Conference on Artificial Intelligence}, AAAI'16, page 2037–2043. AAAI Press.

\bibitem[Streeter and Golovin, 2007]{streeter_online_2007}
Streeter, M. and Golovin, D. (2007).
\newblock An {Online} {Algorithm} for {Maximizing} {Submodular} {Functions}:.
\newblock Technical report, Defense Technical Information Center, Fort Belvoir, VA.

\bibitem[Streeter and Golovin, 2008]{streeter_online_2008}
Streeter, M. and Golovin, D. (2008).
\newblock An {Online} {Algorithm} for {Maximizing} {Submodular} {Functions}.
\newblock In {\em Advances in {Neural} {Information} {Processing} {Systems}}, volume~21. Curran Associates, Inc.

\bibitem[Streeter et~al., 2009]{NIPS2009_e0c64119}
Streeter, M., Golovin, D., and Krause, A. (2009).
\newblock Online learning of assignments.
\newblock In Bengio, Y., Schuurmans, D., Lafferty, J., Williams, C., and Culotta, A., editors, {\em Advances in Neural Information Processing Systems}, volume~22. Curran Associates, Inc.

\bibitem[Sviridenko et~al., 2014]{sviridenko_optimal_2014}
Sviridenko, M., Vondrák, J., and Ward, J. (2014).
\newblock Optimal approximation for submodular and supermodular optimization with bounded curvature.
\newblock arXiv:1311.4728 [cs].

\bibitem[Svitkina and Fleischer, 2010]{svitkina_submodular_2010}
Svitkina, Z. and Fleischer, L. (2010).
\newblock Submodular approximation: sampling-based algorithms and lower bounds.
\newblock arXiv:0805.1071 [cs].

\bibitem[Wan et~al., 2023]{pmlr-v202-wan23e}
Wan, Z., Zhang, J., Chen, W., Sun, X., and Zhang, Z. (2023).
\newblock Bandit multi-linear {DR}-submodular maximization and its applications on adversarial submodular bandits.
\newblock In Krause, A., Brunskill, E., Cho, K., Engelhardt, B., Sabato, S., and Scarlett, J., editors, {\em Proceedings of the 40th International Conference on Machine Learning}, volume 202 of {\em Proceedings of Machine Learning Research}, pages 35491--35524. PMLR.

\bibitem[Wen et~al., 2017]{10.5555/3294996.3295062}
Wen, Z., Kveton, B., Valko, M., and Vaswani, S. (2017).
\newblock Online influence maximization under independent cascade model with semi-bandit feedback.
\newblock In {\em Proceedings of the 31st International Conference on Neural Information Processing Systems}, NIPS'17, page 3026–3036, Red Hook, NY, USA. Curran Associates Inc.

\bibitem[Zhang et~al., 2019]{NEURIPS2019_b43a6403}
Zhang, M., Chen, L., Hassani, H., and Karbasi, A. (2019).
\newblock Online continuous submodular maximization: From full-information to bandit feedback.
\newblock In Wallach, H., Larochelle, H., Beygelzimer, A., d\textquotesingle Alch\'{e}-Buc, F., Fox, E., and Garnett, R., editors, {\em Advances in Neural Information Processing Systems}, volume~32. Curran Associates, Inc.

\bibitem[Zhang et~al., 2022]{pmlr-v162-zhang22e}
Zhang, Q., Deng, Z., Chen, Z., Hu, H., and Yang, Y. (2022).
\newblock Stochastic continuous submodular maximization: Boosting via non-oblivious function.
\newblock In Chaudhuri, K., Jegelka, S., Song, L., Szepesvari, C., Niu, G., and Sabato, S., editors, {\em Proceedings of the 39th International Conference on Machine Learning}, volume 162 of {\em Proceedings of Machine Learning Research}, pages 26116--26134. PMLR.

\bibitem[Zhu et~al., 2021]{JMLR:v22:18-407}
Zhu, J., Wu, Q., Zhang, M., Zheng, R., and Li, K. (2021).
\newblock Projection-free decentralized online learning for submodular maximization over time-varying networks.
\newblock {\em Journal of Machine Learning Research}, 22(51):1--42.

\end{thebibliography}

\newpage
\appendix

\section{Lowerbound proofs}
\subsection{Proof of Theorem \ref{thm:main}}
\label{proof:lower}

    For any $\{x_1, x_2, \ldots, x_k\} \in \binom{[n]\setminus \{1,\dots,k\}}{k}$, define instance $\mc{H}_0, \mc{H}_{(x_1, \ldots, x_k)}, \mc{H}_{(x_{i + 1}, \ldots, x_k)}$ with reward functions as follows:
\begin{align*}
f_{\mc{H}_0}(S) &:= \begin{cases}
    H_{|S| + k} - H_k = \sum_{i = 1} ^{|S|} \frac{1}{k + i} & S = \{1, 2, \ldots, |S|\} \\
    H_{|S| + k} - H_k - \Delta & \text{Otherwise}
\end{cases}
\\
f_{\mc{H}_{(x_1, \ldots, x_k)}}(S) &:= \begin{cases}
        H_{|S| + k} - H_k + \Delta & S = \{x_1, x_2, \ldots, x_{|S|}\} \\ 
    H_{|S| + k} - H_k & S = \{1, 2, \ldots, |S|\} \\
    H_{|S| + k} - H_k - \Delta & \text{Otherwise}
\end{cases}
\\
f_{\mc{H}_{(x_{i + 1}, \ldots, x_k)}}(S) &:= \begin{cases}
    H_{|S| + k} - H_k + \Delta & S = \{1, \ldots, i, x_{i + 1}, \ldots, x_{|S|}\} \\ 
    H_{|S| + k} - H_k & S = \{1, 2, \ldots, |S|\} \\
    H_{|S| + k} - H_k - \Delta & \text{Otherwise}
\end{cases}
\end{align*}
where $H_n = \sum_{k=1}^n \frac{1}{k}$ is the $n$-th harmonic number. %\lalit{Define appropriate $S^{\ast}$?}

\begin{lemma}
    If $\Delta \le  (1/{8k^2})$ then $\mc{H}_0$ and $\mc{H}_{(x_1, \ldots, x_k)}$ are submodular.
\end{lemma}
\begin{proof}
    for any $S \subsetneq T \subset [n]$ where $|T| < k$ (the function is only defined on sets of cardinality at most $k$) and $x \notin T$ we have to show $f(S+x) - f(S) \ge f(T+x) - f(T)$.
    \begin{align*}
    &f(T + x) - f(T) \le \frac{1}{|T| + k + 1} + 2 \Delta \le \frac{1}{|T| + k + 1} + \frac{1}{4k^2} \le \frac{1}{|T| + k} - \frac{1}{4k^2} \le \frac{1}{|T| + k} - 2\Delta \\ &\le \frac{1}{|S| + 1 + k} - 2\Delta \le f(S + x) - f(S)
    \end{align*}
    
\end{proof}
 For $\mc{H}_0$ if $\epsilon_i < \Delta$ at each step $i$ of the greedy arm selection, then $S_{\textbf{gr}}^{k, \boldsymbol{\epsilon}} = \{1, \ldots, k\}$, otherwise $f_{\mc{H}_0}(S_{\textbf{gr}}^{k, \boldsymbol{\epsilon}}) + \mb{1}^T\boldsymbol{\epsilon} \ge H_{2k} - H_k + \Delta - \Delta = H_{2k} - H_k = f_{\mc{H}_0}(\{1, \ldots, k\})$. So $\min_{\boldsymbol{\epsilon}} f_{\mc{H}_0}(S_{\textbf{gr}}^{k, \boldsymbol{\epsilon}}) + \mb{1}^T\boldsymbol{\epsilon} = f_{\mc{H}_0}(\{1, \ldots, k\})$. This means that we can compute our regret against $f_{\mc{H}_0}(\{1, \ldots, k\})$. Similarly, $\min_{\boldsymbol{\epsilon}} f_{\mc{H}_(x_1, \ldots, x_k)}(S_{\textbf{gr}}^{k, \epsilon}) + \mb{1}^T\boldsymbol{\epsilon} = H_{2k} - H_k + \Delta = f_{\mc{H}_(x_1, \ldots, x_k)}(\{x_1, \ldots, x_k\})$ showing that we can compute our regret against $\{x_1, \cdots, x_k\}$.

Let $\E_{0}$ and $\E_{(x_1,\dots,x_k)}$ denote the probability law under $\mc{H}_0$ and $\mc{H}_{(x_1,\dots,x_k)}$, respectively.
For any $S \subset [n]$ let $T_S$ denote the random variable describing the number of time the set $S$ is played by a policy $\pi$.
Define $T_i := \sum_{S \subset [n] : |S|=i} T_S$.
% Let $T_i$ be the random variable of the number of super-arms of cardinality $i$ that a policy $\pi$ plays.

Then by the definition of $\mc{H}_0$ we have %\lalit{Fix subscript on $f$}
\begin{align*}
    \E_0[\reg]
    &\ge \sum_{i = 1}^{k - 1} (f_{\mc{H}_0} (1, \ldots, k) - \max_{S: |S| = i}f_{\mc{H}_0}(S))\E_0[T_i] + \sum_{S: |S| = k} (f_{\mc{H}_0}(\{1, \ldots, k\}) - f_{\mc{H}_0}(S))\E_0[T_S]
    \\ &\ge \sum_{i=1}^{k-1} \Big(\sum_{j = i+1}^{k} 1/(k+j)\Big)\E_0[T_i] + \Delta \sum_{\{y_1, \ldots, y_k\}\ne\{1, \ldots, k\}} \E_0[T_{\{y_1, \ldots, y_k\}}] 
    \\ &\ge \sum_{i=1}^{k-1} \frac{k-i}{2k}\E_0[T_i] + \frac{\Delta T}{2}\P_0(T_{\{1, \ldots, k\}} \le T/2)
\end{align*}
Similarly for $\mc{H}_{(x_1, \ldots, x_k)}$ we have
\begin{align*}
    &\E_{\{x_1, \ldots, x_k\}}[\reg]     \\ &\ge \sum_{i = 1}^{k - 1} (f_{\mc{H}_{(x_1, \ldots, x_k)}}(\{x_1, \ldots, x_k\}) - \max_{|S| = i}f_{\mc{H}_{(x_1, \ldots, x_k)}}(S))\E_{(x_1, \ldots, x_k)}[T_i] \\ &+ \sum_{S: |S| = k} (f_{\mc{H}_{(x_1, \ldots, x_k)}}(\{x_1, \ldots, x_k\}) - f_{\mc{H}_{(x_1, \ldots, x_k)}}(S))\E_0[T_S]
    \\ &\ge \sum_i^{k-1} (\sum_{j = i+1}^{k} 1/(k+j))\E_{\{x_1, \ldots, x_k\}}[T_i] + \Delta \sum_{\{y_1, \ldots, y_k\}\ne\{x_1, \ldots, x_k\}} \E_{\{x_1, \ldots, x_k\}}[T_{\{y_1, \ldots, y_k\}}] 
    \\ &\ge \frac{\Delta T}{2} \P_{\{x_1, \ldots, x_k\}}(T_{\{1, \ldots, k\}} > T/2).
\end{align*}

% \begin{align*}
%     \E_0[\mathcal{R}_T] &\ge \sum_i^{k-1} (1/(i+1) - 1/(k+1))\E_0[T_i] + \Delta \sum_{\{y_1, \ldots, y_k\}\ne\{1, \ldots, k\}} \E_0[T_{\{y_1, \ldots, y_k\}}] 
%     \\ &\ge \frac{1}{(k+1)^2}\sum^{k-1} \E_0[T_i] + \frac{\Delta T}{2}\P_0(T_{\{1, \ldots, k\}} \le T/2)
% \end{align*}

% \begin{align*}
%     \E_{(x_1, \ldots, x_k)}[\mathcal{R}_T] &\ge \sum_i^{k-1} (1/(i+1) - 1/(k+1))\E_{(x_1, \ldots, x_k)}[T_i] + \Delta \sum_{(y_1, \ldots, y_k)\ne(x_1, \ldots, x_k)} \E_{(x_1, \ldots, x_k)}[T_{(y_1, \ldots, y_k)}] 
%     \\ &\ge \frac{\Delta T}{2} \P_0(T_{\{1, \ldots, k\}} > T/2)
% \end{align*}

\begin{lemma}
\label{lemma:pigeon}
    For any $i \le k$ here exist a sequence $(x_i, \ldots ,x_k)$, where 
    \begin{align*}    
    & \sum_{j=i}^{k}\E_0[T_{\{1, \ldots, i - 1, x_i, \ldots, x_j\}}]  \\ &\le \frac{1}{n - k} \E_0[T_i] + \frac{2}{(n - k)(n - k - 1)}\E_0[T_{i + 1}] + \frac{4}{(n - k)(n - k - 1)} \sum_{j = i + 2}^{k-1} \frac{k - j}{2k} \E_0[T_j] + \frac{T}{{n - k \choose k -i + 1}}.\end{align*}
\end{lemma}
\begin{proof}
    For $i \le k$ and a sequence $(x_i,...,x_k)$, define $Q_{(x_i,...,x_k)} := \sum_{j = i}^{k} \E_0[T_{\{1, \ldots, i - 1, x_i, \ldots, x_j\}}] $. Then we have 
    \[
    Q := \sum_{(x_i, \ldots, x_k) \ne (i, \ldots, k)} Q_{(x_i,...,x_k)} \le \sum_{j = i}^{k - 1} \frac{(n - k - j + i - 1)! (j - i + 1)!}{(n - 2k + i - 1)!} \E_0[T_j] + ((k - i + 1)!)\E_0[T_k].
    \]
    Then by Pigeonhole principle, the exists a sequence $(x_i, \ldots, x_k)$ such that 
    \begin{align*}        
    Q_{(x_i,...,x_k)} &\le \frac{Q}{\frac{(n - k)!}{(n - 2k + i - 1)!}} \\ &\le
    \sum_{j=i}^{k - 1} \frac{(n - k - j + i - 1)!(j - i + 1)!}{(n - k)!} \E_0[T_j] + \frac{(n - 2k + i - 1)!(k - i + 1)!}{(n - k)!} \E[T_k] \\ &\le \frac{1}{n - k}  \E_0[T_i] + \frac{1}{(n - k)(n - k - 1)}\sum_{j = i + 1}^{k - 1} \frac{(j - i)(j - i - 1)}{{n - k - 2 \choose j - i - 2}} \E[T_j] + \frac{1}{{n - k \choose k - i + 1}} \E[T_k] \\ &\le \frac{1}{n - k} \E_0[T_i] + \frac{2}{(n - k)(n - k - 1)}\E_0[T_{i + 1}] \\ &\quad + \frac{4}{(n - k)(n - k - 1)} \sum_{j = i + 2}^{k-1} \frac{k - j}{2k} \E_0[T_j] + \frac{T}{{n - k \choose k -i + 1}}.
    \end{align*}

\end{proof}

% \begin{align*}
%     &2\max\Big(\E_0[\mathcal{R}_T], \max_{(x_1, \ldots, x_k)\ne(1, \ldots, k)} \E_{\{x_1, \ldots, x_k\}}[\mathcal{R}_T] \Big)
%      \\ 
%     &\ge \max_{(x_1, \ldots, x_k)\ne(1, \ldots, k)} \sum_i^{k-1} \frac{k-i}{2k}\E_0[T_i] + \frac{\Delta T}{2} \exp\Big(-2 \Delta^2 \sum_{i=1}^k \E[T_{\{x_1, \ldots, x_i\}}]\Big) 
%     \\
%     &\ge \min_{\lambda \in \Delta_n} \min_{\alpha_i \in \Delta_{n \choose i},i=1}^{k} \max_{(x_1, \ldots, x_k)\ne(1, \ldots, k)} \sum_i^{k-1} \frac{k-i}{2k}\lambda_i T + \frac{\Delta T}{2} \exp\Big(-2 \Delta^2 \sum_{i=1}^k \lambda_i \alpha_i \Big) 
% \end{align*}

    % $\mathbf{512k^4n \le T \le \frac{16}{n^2k^2}{n - k \choose k}^3}$ :

% Let $i^* \le k$ be the integer such that $n^{3i^* - 2}\le T \le n^{3(i^* + 1) - 2}$. If $i^* > k$, then we have regret $\mathcal{O}(T^{1/2}{n - k \choose k})^{1/2}$.
\begin{lemma}
    \label{lemma:KL}
    For $\mc{H}_0$ and $\mc{H}_{(x_1, \ldots, x_k)}$ defined above, we have
    \[
    KL(\P_0|\P_{\{x_i, \ldots, x_k\}}) = 2\Delta^2\sum_{j = i}^{k - 1} \E_0[T_{1, \ldots, i - 1, x_i, \ldots, x_j}]
    \]
\end{lemma}

\begin{proof}
    \begin{align*}
        KL(\P_0|\P_{\{x_i, \ldots, x_k\}}) &= \sum_{S: |S| \le k} \E_0[T_S] KL(P_0(S) | P_{\{x_i, \ldots, x_k\}}(S)) \tag{lemma 15.1 in \cite{bandit_book}} \\ &= \sum_{j = i}^{k} 2\Delta^2 \E_0[T_{1, \ldots, i - 1, x_i, \ldots, x_j}]  
    \end{align*}

    where $P_0(S) = \mc{N}(f_{{\mc{H}}_{0}}(S), 1)$ and $P_{\{x_i, \ldots, x_k\}}(S)
    = \mc{N}(f_{\mc{H}_{(x_i, \ldots, x_k)}}(S), 1)$
    are the reward distributions of arm $S$ in $\mc{H}_0$ and $\mc{H}_{(x_i, \ldots, X_k)}$ respectively.
    
\end{proof}
Using two above lemmas, we have,
\begin{align*}
    &2\max\Big(\E_0[\reg], \max_{1 \le i \le k, (x_i, \ldots, x_k)\ne(i, \ldots, k)} \E_{\{x_i, \ldots, x_k\}}[\reg] \Big) 
    \\
    &\ge 
    \max_{1 \le i \le k, (x_i, \ldots, x_k)\ne(i, \ldots, k)} E_0[\reg] + \E_{\{x_i, \ldots, x_k\}}[\reg] 
    \\
    &\ge \max_{1 \le i \le k, (x_i, \ldots, x_k)\ne(i, \ldots, k)} \frac{\Delta T}{2}\Big( \P_0(T_{\{1, \ldots, k\}} \le T/2) + \P_{\{x_i, \ldots, x_k\}}(T_{\{1, \ldots, k\}} > T/2) \Big)\\
    &\ge \max_{1 \le i \le k, (x_i, \ldots, x_k)\ne(i, \ldots, k)} \frac{\Delta T}{2} \exp(-KL(\P_0|\P_{\{x_i, \ldots, x_k\}}))\tag{Using Pinsker's Inequality~\cite{bandit_book}}
    \\
    &\ge \max_{1 \le i \le k, (x_i, \ldots, x_k)\ne(i, \ldots, k)} \frac{\Delta T}{2} \exp\Big(-2 \Delta^2 \sum_{j=i}^{k} \E_0[T_{\{1, \ldots, i - 1, x_i, \ldots, x_j\}}] \Big) \tag{Using lemma \ref{lemma:KL}}
    \\ 
    &\ge \frac{\Delta T}{2} \max_{1\leq i\leq k}\exp\Big(-2 \Delta^2 (\frac{1}{n - k} \E_0[T_i] + \frac{2}{(n - k)(n - k - 1)}\E_0[T_{i + 1}] \\ & \quad + \frac{4}{(n - k)(n - k - 1)} \sum_{j = i + 2}^{k-1} \frac{k - j}{2k} \E_0[T_j] + \frac{T}{{n - k \choose k -i + 1}} ) \Big)\tag{Using Lemma \ref{lemma:pigeon}} \\ 
    &\ge \max_{1 \le i \le k} \frac{1}{2} (k - i^*)^{1/3}T^{2/3}n^{1/3} \exp\Big(-2 T^{-2/3}(k- i^*)^{2/3}n^{2/3} ( \frac{1}{n - k} \E_0[T_i] \\
    & \quad + \frac{2}{(n - k)(n - k - 1)}\E_0[T_{i + 1}] + \frac{4}{(n - k)(n - k - 1)} \sum_{j = i + 2}^{k-1} \frac{k - j}{2k} \E_0[T_j] + \frac{T}{{n - k \choose k -i + 1}}))\Big) \tag{Setting $\Delta := ((k - i^*)n/T)^{1/3}$}
\end{align*}

For $1 \le i \le k - i^* + 1$, $\frac{n^{2/3}T^{1/3}}{{n - k \choose k - i + 1}} \le 1$ by definition of $i^*$; so either the maximum regret is larger than $\frac{1}{4} T^{2/3}(k - i^*)^{1/3}n^{1/3}\exp(-8)$, which proves the theorem, or $  \frac{1}{n - k} \E_0[T_i] + \frac{2}{(n - k)(n - k - 1)}\E_0[T_{i + 1}] + \frac{4}{(n - k)(n - k - 1)} \sum_{j = i + 2}^{k-1} \frac{k - j}{2k} \E_0[T_j] \ge 3(1/\Delta^2)$. If the third term is larger than $1/\Delta^2$, then $\sum_{j=i + 2}^{k-1}\frac{k - j}{2k} \E_0[T_j] \ge \frac{1}{16} \frac{n}{(k - i^*)^{2/3}} T^{2/3}n^{1/3}$ which proves the lowerbound as $\frac{n}{(k - i^*)^{2/3}} \ge (k - i^*)^{1/3}$. Therefore, the only remaining case is that either the first or second term is $\ge 1/\Delta^2$. 
This means that for $1 \le i \le k - i^* + 1$, either $\E_0[T_i] \ge \frac{1}{4}(k - i^*)^{-2/3}T^{2/3}n^{1/3}$ or $\E_0[T_{i + 1}] \ge \frac{1}{8}(n - k - 1)(k - i^*)^{-2/3}T^{2/3}n^{1/3} \ge \frac{1}{4}(k - i^*)^{-2/3}T^{2/3}n^{1/3}$.
Therefore, for at least half of the $1 \le i \le k - i^* + 1$, $\E_0[T_i] \ge \frac{1}{4}(k - i^*)^{-2/3}T^{2/3}n^{1/3}$, and
\[
\E_0[\reg] \ge
\sum_{j = 1}^{k - i^* + 1} \frac{k - j}{2k} \E_0[T_j] \ge \frac{1}{8}(k - i^*)^{1/3}T^{2/3}n^{1/3}.
\]

Note that since $T \ge 512k^7n$, we have $\Delta \le (k n/T)^{1/3} \le \frac{1}{8k^2}$, so the functions with this selection of $\Delta$ are submodular. 

% $\mathbf{T \ge \frac{16}{(n - k)^2}{n - k \choose k}^3}$ : 

We now lower bound the regret in a different way. Let $\lambda := \frac{\sum_{j=k - i^* + 1}^{k-1} \frac{k-i}{2k}\E_0[T_i]}{T}$, then $\lambda \le 1$, and using lemma \ref{lemma:pigeon} we have that there exists a selection of $(x_i, \ldots, x_k)$ such that,
\begin{align*}    
    &\sum_{j=i}^{k}\E_0[T_{\{1, \ldots, i - 1, x_i, \ldots, x_j\}}]  \\ &\le \frac{1}{n - k} \E_0[T_i] + \frac{2}{(n - k)(n - k - 1)}\E_0[T_{i + 1}] + \frac{4}{(n - k)(n - k - 1)} \sum_{j = i + 2}^{k-1} \frac{k - j}{2k} \E_0[T_j] + \frac{T}{{n - k \choose k -i + 1}} \\ &\le \frac{4}{n - k} \sum_{j = i}^{k - 1}\frac{k - j}{2k}\E_0[T_j] + \frac{T}{{n - k \choose k -i + 1}} =  \frac{4}{(n - k)} \lambda T + \frac{T}{{n - k \choose k - i + 1}}
\end{align*}
So
\begin{align*}
    &2\max\Big(\E_0[\reg], \max_{1 \le i \le k, (x_i, \ldots, x_k)\ne(i, \ldots, k)} \E_{\{x_i, \ldots, x_k\}}[\reg] \Big) \\ &\ge 
    \max_{1 \le i \le k, (x_i, \ldots, x_k)\ne(i, \ldots, k)} E_0[\reg] + \E_{\{x_i, \ldots, x_k\}}[\reg] 
    \\
    &\ge \min_{\lambda \in [0, 1]} \max_{1 \le i \le k, (x_i, \ldots, x_k)\ne(i, \ldots, k)} \lambda T + \frac{\Delta T}{2} \exp\Big(-2 \Delta^2 \sum_{j=i}^{k} \E_0[T_{\{1, \ldots, i - 1, x_i, \ldots, x_j\}}]  \Big) 
    \\ 
    &\ge \min_{\lambda \in [0, 1]} \max_{1 \le i \le k} \lambda T + \frac{\Delta T}{2} \exp\Big(-2 \Delta^2 \big( \frac{4}{n - k} \lambda T + \frac{T}{{n - k \choose k - i + 1}} \big) \Big) \\ 
    &\ge \min_{\lambda \in [0, 1]} \max_{1 \le i \le k - i^* - 1} \lambda T + \frac{1}{2} T^{1/2}{n - k\choose k - i + 1}^{1/2} \exp\Big(-2\frac{4\lambda{n - k\choose k - i + 1}}{(n - k)} - 2\Big) \tag{Setting $\Delta := ({n - k\choose k - i + 1}/T)^{1/2}$} \\ & \ge \frac{1}{2}T^{1/2}{n - k \choose i^*}^{1/2}e^{-2}
\end{align*}

% \begin{align*}
%         2\max\Big(&\E_0[\mathcal{R}_T], \max_{(x_1, \ldots, x_k)\ne(1, \ldots, k)} \E_{\{x_1, \ldots, x_k\}}[\mathcal{R}_T] \Big) \\ 
%         &\ge 
%     \max_{(x_1, \ldots, x_k)\ne(1, \ldots, k)} E_0[\mathcal{R}_T] + \E_{\{x_1, \ldots, x_k\}}[\mathcal{R}_T] 
%     \\
%     &\ge \min_{\lambda \in [0, 1]} \max_{(x_1, \ldots, x_k)\ne(1, \ldots, k)} \lambda T + \frac{\Delta T}{2} \exp\Big(-2 \Delta^2 \sum_{i=1}^k \E[T_{\{x_1, \ldots, x_i\}}]\Big) 
%     \\ 
%     &\ge \min_{\lambda \in [0, 1]} \lambda T + \frac{\Delta T}{2} \exp\Big(-2 \Delta^2 (\frac{4}{n - k}\lambda T + \frac{T}{{n - k \choose k}}) \Big)
%     \\ 
%     &\ge_{\Delta := ({n - k\choose k}/T)^{1/2}} \min_{\lambda \in [0, 1]} \lambda T + \frac{1}{2} T^{1/2}{n - k \choose k}^{1/2} \exp\Big(-2(\frac{4 {n - k \choose k} \lambda} {n - k} + 1)\Big)
%     \\
%     &\ge \frac{1}{2} T^{1/2}{n - k \choose k}^{1/2} \exp(-2)
% \end{align*}
The last inequality holds as $\log( \frac{4 T^{1/2} {n - k \choose k - i + 1}^{3/2}}{(n - k)T}) \le 0$, and the function relative to $\lambda$ is convex, $\lambda = 0$ minimizes in the last inequality.
Combining the two parts of the proof we have
\begin{align*}
    &\max\Big(\E_0[\reg], \max_{1 \le i \le k, (x_i, \ldots, x_k)\ne(i, \ldots, k)} \E_{\{x_i, \ldots, x_k\}}[\reg] \Big) \\ &\ge \max\big(\frac{1}{8}(k - i^*)^{1/3}T^{2/3}n^{1/3}e^{-8}, \frac{1}{2}T^{1/2}{n - k \choose i^*}^{1/2}e^{-2} \big) \\ &\ge \frac{1}{16}(k - i^*)^{1/3}T^{2/3}n^{1/3}e^{-8} + \frac{1}{4}T^{1/2}{n - k \choose i^*}^{1/2}e^{-2} 
\end{align*}

\subsection{Proof of Theorem \ref{thm:naet-lower}}

    We generalize the lowerbound distance of Theorem \ref{thm:main} by having the gap $\Delta_i$ in cardinality $i$.
    For any $\{x_1, x_2, \ldots, x_k\} \in \binom{[n]\setminus \{1,\dots,k\}}{k}$, define instance $\mc{H}_0, \mc{H}_{(x_1, \ldots, x_k)}, \mc{H}_{(x_{i + 1}, \ldots, x_k)}$ with reward functions as follows:
\begin{align*}
f_{\mc{H}_0}(S) &:= \begin{cases}
    H_{|S| + k} - H_k = \sum_{i = 1} ^{|S|} \frac{1}{k + i} & S = \{1, 2, \ldots, |S|\} \\
    H_{|S| + k} - H_k - \Delta_{|S|} & \text{Otherwise}
\end{cases}
\\
f_{\mc{H}_{(x_1, \ldots, x_k)}}(S) &:= \begin{cases}
        H_{|S| + k} - H_k + \Delta_{|S|} & S = \{x_1, x_2, \ldots, x_{|S|}\} \\ 
    H_{|S| + k} - H_k & S = \{1, 2, \ldots, |S|\} \\
    H_{|S| + k} - H_k - \Delta_{|S|} & \text{Otherwise}
\end{cases}
\\
f_{\mc{H}_{(x_{i + 1}, \ldots, x_k)}}(S) &:= \begin{cases}
    H_{|S| + k} - H_k + \Delta_{|S|} & S = \{1, \ldots, i, x_{i + 1}, \ldots, x_{|S|}\} \\ 
    H_{|S| + k} - H_k & S = \{1, 2, \ldots, |S|\} \\
    H_{|S| + k} - H_k - \Delta_{|S|} & \text{Otherwise}
\end{cases}
\end{align*}
The KL divergance between reward distribution of two instances is similarly: 
\begin{align*}
    KL(\P_0|\P_{\{x_i, \ldots, x_k\}}) &= \sum_{j = i}^{k} 2\Delta_j^2 \E_0[T_{1, \ldots, i - 1, x_i, \ldots, x_j}]  
\end{align*}

\begin{lemma}
\label{lemma:na-pigeon}
    For any $i \le k$ here exist a sequence $(x_i, \ldots ,x_k)$, where 
    \begin{align*}    
    \sum_{j=i}^{k} \Delta_j^2 \E_0[T_{\{1, \ldots, i - 1, x_i, \ldots, x_j\}}]  &\le \frac{1}{n - k} \Delta_i^2 \E_0[T_i] + \frac{2}{(n - k)(n - k - 1)}\Delta_{i + 1}^2\E_0[T_{i + 1}] \\ &+ \frac{6}{(n - k)(n - k - 1)(n - k - 2)} \Delta_{i + 2}^2  \\ &+ \frac{12}{(n - k)(n - k - 1)(n - k - 2)} \sum_{j = i + 3}^{k-1} \frac{k - j}{2k} \Delta_{j}^2 \E_0[T_j] + \frac{\Delta_k^2 T}{{n - k \choose k -i + 1}}.\end{align*}
\end{lemma}
\begin{proof}
    For $i \le k$ and a sequence $(x_i,...,x_k)$, define $Q_{(x_i,...,x_k)} := \sum_{j = i}^{k} \Delta_j^2 \E_0[T_{\{1, \ldots, i - 1, x_i, \ldots, x_j\}}] $. Then we have 
    \[
    Q := \sum_{(x_i, \ldots, x_k) \ne (i, \ldots, k)} Q_{(x_i,...,x_k)} \le \sum_{j = i}^{k - 1} \frac{(n - k - j + i - 1)! (j - i + 1)!}{(n - 2k + i - 1)!} \Delta_j^2 \E_0[T_j] + ((k - i + 1)!)\Delta_k^2 \E_0[T_k].
    \]
    Then by Pigeonhole principle, the exists a sequence $(x_i, \ldots, x_k)$ such that 
    \begin{align*}        
    Q_{(x_i,...,x_k)} &\le \frac{Q}{\frac{(n - k)!}{(n - 2k + i - 1)!}} \\ &\le
    \sum_{j=i}^{k - 1} \frac{(n - k - j + i - 1)!(j - i + 1)!}{(n - k)!} \Delta_j^2 \E_0[T_j] + \frac{(n - 2k + i - 1)!(k - i + 1)!}{(n - k)!} \Delta_k^2 \E[T_k] \\ &\le \frac{1}{n - k} \Delta_i^2  \E_0[T_i] + \frac{1}{(n - k)(n - k - 1)}\sum_{j = i + 1}^{k - 1} \frac{(j - i)(j - i - 1)}{{n - k - 2 \choose j - i - 2}} \Delta_j^2 \E[T_j] + \frac{1}{{n - k \choose k - i + 1}} \E[T_k] \\ &\le \frac{1}{n - k} \Delta_i^2 \E_0[T_i] + \frac{2}{(n - k)(n - k - 1)} \Delta_{i _ 1}^2 \E_0[T_{i + 1}] \\ &\quad + \frac{6}{(n - k)(n - k - 1)(n - k - 2)} \Delta_{i + 2}^2  \\ & \quad + \frac{12}{(n - k)(n - k - 1)(n - k - 2)} \sum_{j = i + 3}^{k-1} \frac{k - j}{2k} \Delta_{j}^2 \E_0[T_j] + \frac{\Delta_k^2 T}{{n - k \choose k -i + 1}}.
    \end{align*}

\end{proof}

We now assign $\Delta_i$ for lower cardinalities based on the value of $\Delta_k$. If $\mathbf{1}^T\boldsymbol{\epsilon}' \le 2 \Delta_k$, For $i \le k - 1$, we assign $\Delta_i = \epsilon'_i$, so a greedy procedure with $\boldsymbol{\epsilon}'$ will retrieve the best set, hence $f_{\mc{H}_0}(S_{\textbf{gr}}^{k, \boldsymbol{\epsilon}'}) + \mb{1}^T\boldsymbol{\epsilon}' \ge f_{\mc{H}_0}(\{1, \ldots, k\})$
and $f_{\mc{H}_{(x_i, \ldots, x_k)}}(S_{\textbf{gr}}^{k, \boldsymbol{\epsilon}'}) + \mb{1}^T\boldsymbol{\epsilon}' \ge f_{\mc{H}_{(x_i, \ldots, x_k)}}(\{1, \ldots, i - 1, x_i \ldots, x_k\})$. Otherwise, since the gap of any set of size $k$ and the best set is at most $2\Delta_k$ for both $\mc{H}_0$ and $\mc{H}_{(x_i, \ldots, x_k)}$, $f_{\mc{H}_0}(S_{\textbf{gr}}^{k, \boldsymbol{\epsilon}'}) + \mb{1}^T\boldsymbol{\epsilon}' \ge H_{2k} - H_k - \Delta_k + 2\Delta_k = f_{\mc{H}_0}(\{1, \ldots, k\})$ and $f_{\mc{H}_{(x_i, \ldots, x_k)}}(S_{\textbf{gr}}^{k, \boldsymbol{\epsilon}'}) + \mb{1}^T\boldsymbol{\epsilon}' \ge  H_{2k} - H_k - \Delta_k + 2\Delta_k \ge f_{\mc{H}_{(x_i, \ldots, x_k)}}(\{1, \ldots, i - 1, x_i \ldots, x_k\})$; so for $i \le k - 1$, and we assign $\Delta_i = \frac{\Delta_k}{k}$. Therefore, in both cases $\reg \ge R(S^*)$, and we give a lower bound for $R(S^*)$.

For the first part of the lower bound, we'll assign $\Delta_k = (k - i^*) (\frac{n}{T})^{1/3}$. 
Now similarly to proof of Theorem \ref{thm:main}, we have

\begin{align*}
    &2\max\Big(\E_0[\reg], \max_{1 \le i \le k, (x_i, \ldots, x_k)\ne(i, \ldots, k)} \E_{\{x_i, \ldots, x_k\}}[\reg] \Big) 
    \\
    &\ge \max_{1 \le i \le k, (x_i, \ldots, x_k)\ne(i, \ldots, k)} \frac{\Delta_k T}{2} \exp\Big(-2 \sum_{j=i}^{k} \Delta_j^2 \E_0[T_{\{1, \ldots, i - 1, x_i, \ldots, x_j\}}] \Big)
    \\ 
    &\ge \frac{\Delta_k T}{2} \max_{1\leq i\leq k}\exp\Big(-2 (\frac{1}{n - k} \Delta_i^2 \E_0[T_i] + \frac{2}{(n - k)(n - k - 1)} \Delta_{i + 1}^2 \E_0[T_{i + 1}] \\ & + \frac{6}{(n - k)(n - k - 1)(n - k - 2)} \Delta_{i + 2}^2  + \frac{12}{(n - k)(n - k - 1)(n - k - 2)} \sum_{j = i + 3}^{k-1} \frac{k - j}{2k} \Delta_{j}^2 \E_0[T_j] \\ &+ \frac{\Delta_k^2 T}{{n - k \choose k -i + 1}} ) \Big)\tag{Using Lemma \ref{lemma:na-pigeon}} \\ 
    &\ge \max_{1 \le i \le k} \frac{1}{2} (k - i^*)T^{2/3}n^{1/3} \exp\Big(-2( \frac{1}{n - k} \Delta_i^2 \E_0[T_i] + \frac{2}{(n - k)(n - k - 1)}\Delta_{i + 1}^2 \E_0[T_{i + 1}] \\
    & + \frac{6}{(n - k)(n - k - 1)(n - k - 2)} \Delta_{i + 2}^2 + \frac{12}{(n - k)(n - k - 1)(n - k - 2)} \sum_{j = i + 3}^{k-1} \frac{k - j}{2k} \Delta_{j}^2 \E_0[T_j] \\ &+ \frac{(k - i^*)^2 n^{2/3} T^{1/3} }{{n - k \choose k -i + 1}}))\Big) \tag{Setting $\Delta_k := ((k - i^*)^3 n/T)^{1/3}$}
\end{align*}

For $1 \le i \le k - i^* + 1$, $\frac{(k - i^*)^2n^{2/3}T^{1/3}}{{n - k \choose k - i + 1}} \le \frac{k^2n^{2/3}T^{1/3}}{{n - k \choose k - i + 1}} \le 1$ by definition of $i^*$; so either the maximum regret is larger than $\frac{1}{4} T^{2/3}(k - i^*)n^{1/3}\exp(-10)$, which proves the theorem, or 

\begin{align*}
&\frac{1}{n - k} \Delta_i^2 \E_0[T_i] + \frac{2}{(n - k)(n - k - 1)}\Delta_{i + 1}^2\E_0[T_{i + 1}] + \frac{6}{(n - k)(n - k - 1)(n - k - 2)} \Delta_{i + 2}^2 \\ &+ \frac{12}{(n - k)(n - k - 1)(n - k - 2)} \sum_{j = i + 3}^{k-1} \frac{k - j}{2k} \Delta_{j}^2 \E_0[T_j] \Delta_j^2 \E_0[T_j] \ge 4
\end{align*}

If the forth term is larger than $1$, then 
\begin{align*}
\Delta_k^2 \sum_{j=i + 3}^{k-1}\frac{k - j}{2k} \E_0[T_j] &\ge
\sum_{j=i + 3}^{k-1}\frac{k - j}{2k} \Delta_j^2 \E_0[T_j] \\ &\ge \frac{(n - k)(n - k- 1)(n - k - 2)}{12} \ge \frac{n^3}{96}
\end{align*}

So $\sum_{j=i + 3}^{k-1}\frac{k - j}{2k} \E_0[T_j] \ge \frac{n^3}{96} \frac{1}{\Delta_k^2} \ge \frac{1}{96} (k - i^*) n^{1/3} T^{2/3}  $ which proves the lower bound. 

Therefore, the only remaining case is that at least one of the first three terms is $\ge 1$. 
This means that for $1 \le i \le k - i^* + 1$, either $\E_0[T_i] \ge \frac{n}{2 \Delta_i^2}$, or $\E_0[T_{i + 1}] \ge \frac{n}{4 \Delta_{i + 1}^2 } (n - k - 1) \ge \frac{n}{4 \Delta_{i + 1}^2}$, or $\E_0[T_{i + 2}] \ge \frac{n}{12 \Delta_{i + 2}^2 } (n - k - 1)(n - k - 2) \ge \frac{n}{12 \Delta_{i + 2}^2}$.

 Therefore, for at least $1/3$ of the $1 \le i \le k - i^* + 1$, $\E_0[T_i] \ge \frac{n}{12 \Delta_i^2}$. Let $I$ be all cardinalities in which this inequality holds(so $|I| \ge \frac{k - i^*}{3}$); since $\sum_{i = 1}^{k - 1}  \Delta_i \le 2\Delta_k$, using Lemma \ref{lm: sqaure-sum}, we have
\begin{align*}
\E_0[\reg] &\ge
\sum_{j = 1}^{k - i^* + 1} \frac{k - j}{2k} \E_0[T_j] \ge 
\sum_{j \in I} \frac{k - j}{2k} \frac{n}{12 \Delta^2_j} \ge
\frac{1}{288} (k - i^*) T^{2/3} n^{1/3}.
\end{align*}

For the second part of the lower bound, using $\Delta_i \le 2\Delta_k$, we have     \[
    KL(\P_0|\P_{\{x_i, \ldots, x_k\}}) \le 8\Delta_k^2\sum_{j = i}^{k - 1} \E_0[T_{1, \ldots, i - 1, x_i, \ldots, x_j}]
    \]
    
, and the rest of the argument follows the proof of \ref{thm:main}.

\section{Proof of Theorem~\ref{thm:regret}}
\label{proof:upper}
    \begin{proof}
    We use the notation $l = k - i^*$ to match our lowerbound, however as $k - i^*$ is arbitrary, it can be used for any other choice of $l$ as well.
    Define the event $G := \bigcap_{i=1}^k \bigcap_{a \in [n] \setminus S^{(i-1)}} \bigcap_{t=1}^T g_{i,a,t}$ where
\begin{align*}
    g_{i,a,t} := \left\{ \Big| \sum_{s \leq t: I_s = S^{(i-1)}\cup \{a\}} (r_s - f(S^{(i-1)}\cup \{a\})) \Big| \leq \sqrt{2   T_{S^{(i-1)}\cup \{a\}}(t) \log( 2 kn T^2 ) } \right\}.
\end{align*}
Now note that if $X_s$ are i.i.d. sub-Gaussian random variables then
\begin{align*}
    \P(G^c) &\leq \sum_{i=1}^k \P\Big( \bigcup_{a \in [n] \setminus S^{(i-1)}} \bigcup_{t=1}^T g_{i,a,t}^c \Big) \\
    &= \sum_{i=1}^k \sum_{S \in \binom{[n]}{i-1}} \P\Big( \bigcup_{a \in [n] \setminus S} \bigcup_{t=1}^T g_{i,a,t}^c | S^{(i-1)} = S \Big) \P( S^{(i-1)} = S) \\
    &\leq \sum_{i=1}^k \sum_{S \in \binom{[n]}{i-1}} \sum_{a \in [n] \setminus S}  \P\Big( \bigcup_{t=1}^T g_{i,a,t}^c | S^{(i-1)} = S \Big) \P( S^{(i-1)} = S) \\
    &\leq \sum_{i=1}^k \sum_{S \in \binom{[n]}{i-1}} \sum_{a \in [n] \setminus S}  \P\Big( \bigcup_{t=1}^T \{ |\sum_{s=1}^t X_s| \geq \sqrt{2   t \log( 2 k n T^2 ) } \} \Big) \P( S^{(i-1)} = S) \\
    &\leq \sum_{i=1}^k \sum_{S \in \binom{[n]}{i-1}} \sum_{a \in [n] \setminus S}  \sum_{t=1}^T \frac{1}{k n T^2} \P( S^{(i-1)} = S) \leq 1/T.
\end{align*}
    % \todol{Can you make what you mean by distance precise? Use the above as a guide}
    % \todol{What is $i$, your algorithm indexes by $l$? Make it clear.}
    
    Let $\mathcal{E}_i$ be the event that the arm selected at the $i$-th step of the algorithm is within  $2\sqrt{\frac{2 \log (2knT^2)}{m}}$ of the best possible arm at that step, i.e.
        \begin{align*}
        \mathcal{E}_i = \left\{ \max_{a \notin S^{(i - 1)}}f(S^{(i - 1)} \cup \{a\}) - f(S^{(i)}) \le 2\sqrt{\frac{2 \log (2knT^2)}{m}}\right\}.
    \end{align*}
    We prove that on event $G$, $\cup_{i\in [k - i^*]} \mc{E}_i$ is true.

    Let $a$ be a sub optimal arm with value more than $2\sqrt{\frac{2  \log (2knT^2)}{m}}$ from the best arm  in the $i$-th step.
    That is, if $a^* := \arg\max_{a'} f(S^{(i)} \cup \{a'\})$ and $\Delta_{S^{(i)},a} := f(S^{(i)} \cup \{a^*\}) - f(S^{(i)} \cup \{a\})$, then assume that $\Delta_{S^{(i)},a} \ge 2\sqrt{\frac{2 \log (2knT^2)}{m}}$. Then on event $G$ and arm $a$ being added in $i$-th step,
    \begin{align*}
        U_a(t) \ge U_{a^*}(t)
        \ge f(S^{(i)} \cup \{a^*\})
        = f(S^{(i)} \cup \{a\}) + \Delta_{S^{(i)},a}
    \end{align*}
    which implies
    \begin{align*}
        &U_a(t) - f(S^{(i)} \cup \{a\}) \ge \Delta_{S^{(i)},a} > 2\sqrt{\frac{2  \log (2knT^2)}{m}} .
    \end{align*}
    But this implies that
    \begin{align*}
        \hat{\mu}_{S^{(i)} \cup \{a\}}- f(S^{(i)} \cup \{a\}) > \sqrt{\frac{2  \log (2knT^2)}{m}} 
    \end{align*}
    which is a contradiction of event $G$. Thus, on event $G$ such an arm cannot be selected.

As UCB in the second part of the algorithm has the regret of $65\sqrt{T {{n - k \choose k - l}} \log{T}} + \frac{32}{15} {{n - k \choose k - l}}$ against  $S^{(k)}$ which is the best size $k$ arm containing $S^{(k - i^*)}$ (see \cite{bandit_book}), on event $G$, it is an upper bound for the regret against the greedy solution were the first $k - i^*$ steps select an $\epsilon$-good arm, and the last $i^*$ steps select the best arm, so on event $G$ the regret can be written against a set in $\mc{S}^{k, \boldsymbol{\epsilon}}$ where
\[\mathbf{1}^T \boldsymbol{\epsilon} = (k - i^*)\epsilon = 2(k - i^*)\sqrt{\frac{2 \log{ (2knT^2)}}{m}}.\]
    % Therefore, the  regret $\reg$ on event $G$ can be written as
Therefore, we upper bound the regret relative to $f(S^{k}) + 2(k - i^*)\sqrt{\frac{2 \log(2knT^2)}{m}}$, as by Lemma \ref{lemma:approx-greedy} it's greater than $\frac{1}{c}(1 - e^{-c})f(S^*)$.
%\kevin{But I'm confused, in the regret computation below, it seems like you're comparing to $(f(S^{k, \epsilon}_{gr}) + (k - i^*)\epsilon)$? Is this because the remaining $i^*$ are $\epsilon_i=0$? If so you must say this.}
Let $T_i$ be the set of times where we pulled a set of cardinality $i$. From the while loop condition in the algorithm, we have $|T_i| \le \sum_{a \notin S^{(i - 1)}} \min\big\{{\frac{1}{\Delta^2_{S^{(i - 1)}, a}}}, m\big\} \le (n + 1 - i)m$ for $i \le k - i^*$. For $\epsilon = 2\sqrt{\frac{2 \log (2knT^2)}{m}}$, we have
    \begin{align*}
        \E[\reg] &\le \P[G^c]T + \E[\reg \mathbf{1}\{G\} ]  \le \frac{1}{T}T + \E[\reg \mathbf{1}\{G\}  ] \\ &\le 1 + \sum_{i=1}^{k - i^*} \sum_{t \in T_i} (f(S^{(k)}) + (k - i^*)\epsilon) - f(S^{(i - 1)} \cup \{a_t\})) + \sum_{t \in T_k} (f(S^{(k)}) + (k - i^*)\epsilon) - f(S_t) \\ 
        &\le 1 + (k - i^*)\epsilon T + mn(k - i^*) f(S^{(k)}) + \sum_{t \in T_k} f(S^{(k)}) - f(S_t)  \tag{$f(S^{(i - 1)} \cup \{a_t\})) \ge 0$} \\ &\le 2T(k - i^*)\sqrt{\frac{2  \log (2knT^2)}{m}} + mn(k - i^*) + 65 \sqrt{T {n \choose i^*} } + \frac{32}{15}\frac{n - k}{i^*} + 1 \\ 
        &\le T^{2/3}n^{1/3}(k - i^*)(\log(2knT^2))^{1/3} + \sqrt{8 }T^{2/3}n^{1/3}(k - i^*)(\log (2knT^2))^{1/3} \\ & \quad + 65 \sqrt{T {n \choose i^*} } + \frac{32}{15}\frac{n - k}{i^*} + 1. \tag{Setting $m = T^{2/3}n^{-2/3}\log^{1/3}(2knT^2)$}
    \end{align*}
    \end{proof}

\section{Auxiliary Lemmas}

\begin{lemma}
    \label{lm: sqaure-sum}
    For any sequence of numbers $a_1, \ldots, a_n$ bounded between $(0, 1]$, If $\sum_i a_i \le C \le 1$, then
    \[\sum_{i = 1}^n \frac{1}{a^2_i} \ge \frac{n^3}{C^2}
    \]
\end{lemma}

\begin{proof}
    If there exists $j, k \in [n]$ such that $a_j < a_k$, then for a new sequence $a'_i = \begin{cases}
        a_i \quad i \notin \{j, k\} \\ \frac{a_j + a_k}{2} \quad i \in \{j, k\}
    \end{cases}$ we have
    \begin{align*}
        \sum a^{-2}_i - \sum a'^{-2}_i &= a_j^{-2} + a_k^{-2} - 2\frac{4}{(a_j + a_k)^2} \\ &= \frac{2 + \overbrace{a_j^2a_k^{-2} + a_j^{-2}a_k^2}^{> 2} + 2 ( \overbrace{a_j^{-1}a_k + a_ja_k^{-1}}^{> 2}) - 8 }{a_j^2 + a_k^2 + 2 a_j a_k} > 0 
    \end{align*}
    Therefore, the infimum value of $\sum a^{-2}_i$ over all such sequences is when all elements are equal, and
    \begin{align*}
        \sum_{i = 1}^n \frac{1}{a^2_i} \ge n \big(\frac{n}{\sum a_i}\big)^2 \ge \frac{n^3}{C^2}.
    \end{align*}

\end{proof}

\end{document}